\documentclass[draftclsnofoot,12pt, onecolumn]{IEEEtran}
\usepackage{cite}
\usepackage{amsmath,amssymb,amsfonts,dsfont}
\usepackage{algorithm}
\usepackage{algorithmic}
\usepackage{graphicx}
\usepackage{subfigure}
\usepackage{textcomp}
\usepackage{xcolor}
\usepackage{booktabs}
\usepackage{breqn}
\usepackage{subeqnarray}
\usepackage{cases}
\usepackage{setspace}
\usepackage{amsthm}
\usepackage{bbm}
\usepackage{bm}
\usepackage{lipsum}
\usepackage{enumitem}

\allowdisplaybreaks[4]

\newtheorem{thm}{Theorem}

\newtheorem{lem}{Lemma}
\newtheorem{prop}{Proposition}
\theoremstyle{remark}

\newtheorem{rem}{Remark}
\newtheorem{ass}{Assumption}

\DeclareMathOperator*{\argmin}{arg\,min}
\def\BibTeX{{\rm B\kern-.05em{\sc i\kern-.025em b}\kern-.08em
    T\kern-.1667em\lower.7ex\hbox{E}\kern-.125emX}}

\begin{document}
\title{Efficient Wireless Federated Learning with Partial Model Aggregation\\
}

\author{Zhixiong~Chen,~\IEEEmembership{Student Member,~IEEE},
Wenqiang~Yi,~\IEEEmembership{Member,~IEEE},\\
Arumugam~Nallanathan,~\IEEEmembership{Fellow,~IEEE},
and Geoffrey Ye Li,~\IEEEmembership{Fellow,~IEEE}
\small
\thanks{Zhixiong Chen, Wenqiang Yi, and Arumugam Nallanathan are with the School of Electronic Engineering and Computer Science, Queen Mary University of London, London, U.K. (emails: \{zhixiong.chen, w.yi, a.nallanathan\}@qmul.ac.uk).}
\thanks{Geoffrey Ye Li is with the Faculty of Engineering, Department of Electrical and Electronic Engineering, Imperial College London, England (e-mail: geoffrey.li@imperial.ac.uk).}
\thanks{Part of this work has been accepted to IEEE International Conference on Communications (ICC), 2023 \cite{chen2022is}.}
}

\maketitle
\vspace{-2cm}
\begin{abstract}
The data heterogeneity across devices and the limited communication resources, e.g., bandwidth and energy, are two of the main bottlenecks for wireless federated learning (FL).
To tackle these challenges, we first devise a novel FL framework with partial model aggregation (PMA).
This approach aggregates the lower layers of neural networks, responsible for feature extraction, at the parameter server while keeping the upper layers, responsible for complex pattern recognition, at devices for personalization.
The proposed PMA-FL is able to address the data heterogeneity and reduce the transmitted information in wireless channels.
Then, we derive a convergence bound of the framework under a non-convex loss function setting to reveal the role of unbalanced data size in the learning performance. On this basis,  we maximize the scheduled data size to minimize the global loss function through jointly optimize the device scheduling, bandwidth allocation, computation and communication time division policies with the assistance of Lyapunov optimization.
%
%
Our analysis reveals that the optimal time division is achieved when the communication and computation parts of PMA-FL have the same power.
We also develop a bisection method to solve the optimal bandwidth allocation policy and use the set expansion algorithm to address the device scheduling policy.
Compared with the benchmark schemes, the proposed PMA-FL improves 3.13\% and 11.8\% accuracy on two typical datasets with heterogeneous data distribution settings, i.e., MINIST and CIFAR-10, respectively.
In addition, the proposed joint dynamic device scheduling and resource management approach achieve slightly higher accuracy than the considered benchmarks, but they provide a satisfactory energy and time reduction: 29\% energy or 20\% time reduction on the MNIST; and 25\% energy or 12.5\% time reduction on the CIFAR-10.
\end{abstract}

\section{Introduction}
Federated learning (FL) is a promising distributed learning approach for protecting data privacy.
In FL, edge devices collaboratively train a model under the orchestration of a parameter server (PS), which only requires local learning models/gradients instead of local private data \cite{10024766}.
FL operations can be divided into two parts, namely the communication part and the computation part \cite{9716792}.
For the communication part, the learning performance is constrained by the limited communication resources, e.g., bandwidth and energy.
For the computation part, the model accuracy is degraded by non-independent and identically distributed (non-IID) data samples.
More specifically, the inadequate wireless resources hinder more devices devoted to the FL training process, thus negatively affecting the convergence speed and learning accuracy \cite{yang2022federated, 9530714}.
Moreover, since the PS aggregates models learned from the different devices, the data heterogeneity presented on different devices may lead to weak generalization ability of the trained global model, even resulting in an unstable training process of FL  \cite{9460016}. Therefore, FL needs well-designed solutions to address these two challenges.
\vspace{-0.4cm}
\subsection{Related Works}
From the communication perspective, efficient resource management and device scheduling schemes can enable additional devices to participate in the FL process and thus enhancing learning performance.
To this end, existing works focus on resource optimization \cite{9461628, 9579038, 9609568, 9264742}, device selection \cite{9207871, 9237168, 9605599, 9210812}, and alternating direction method
of multipliers to reduce the communication rounds of training \cite{zhou2021communication}.
The energy-efficient workload partitioning scheme in \cite{9461628} balances the computation between the central processing unit and graphics processing unit in the FL system.
The time-sharing-based transmission scheme in \cite{9579038} can improve the communication efficiency of FL.
In \cite{9609568}, a sequential transmission scheme has been developed for global model aggregation. Based on this transmission scheme, the authors proposed a device heterogeneity-aware scheduling approach to maximize the number of scheduled data samples under energy constraints.
The work in \cite{9264742} introduced  an energy-efficient transmission and computation resource allocation approach for energy consumption minimization of FL system under a latency constraint.
The joint device scheduling and resource allocation policy in \cite{9207871} maximizes the model accuracy in latency-constrained FL.
The joint client selection and bandwidth allocation scheme in \cite{9237168} maximizes the scheduled data samples under long-term client energy constraints.
In \cite{9605599}, a gradient norm approximation method can assist the device scheduling for boosting the training performance in the over-the-air FL system.
A joint learning, wireless resource allocation, and user selection problem has been investigated in \cite{9210812} to minimize an FL loss function.
Although these works have devised different device scheduling and resource management policies to facilitate FL, the joint optimization of communication and computation in FL has been rarely explored.

From the computation perspective, the emerging personalized FL techniques are promising to tackle the data heterogeneity-related challenges, which adapt the collaboratively learned global model for individual clients. Most personalized federated learning techniques involve two steps: 1) devices train a global model in a collaborative fashion, 2) each device personalizes the global model using its private data.
Existing works toward this direction utilize various techniques to implement model personalization in the latter step,
including multi-task learning \cite{NIPS2017_6211080f},
meta-learning \cite{9681911}, 
and model regularization \cite{9187874}. 
More specifically, it has been shown in \cite{NIPS2017_6211080f} that multi-task learning is a natural choice for building personalized federated models. However, the multi-task FL heavily relies on the full participation of devices in each round.
The federated meta-learning algorithm in \cite{9681911} can improve the model accuracy of FL, which maps the meta-training to the federated training process and meta-testing to FL personalization.
A proximal term is introduced in \cite{9187874} to limit the impact of local updates, achieving convergence stability and improving model generalization.
However, the above techniques require more computation or memory resources than the conventional FL algorithms that solely train a global model, e.g., Federated Averaging (FedAvg) \cite{pmlr-v54-mcmahan17a}.

\subsection{Motivations and Contributions}
Although the resource allocation and device scheduling schemes in \cite{9461628, 9579038, 9609568, 9264742, 9207871, 9237168, 9605599, 9210812, zhou2021communication} effectively alleviate the communication burden for FL in wireless networks, they are all operated by averaging local models for global aggregation and are hard to cope with the data heterogeneity nature of FL. In addition, the personalized FL algorithms in \cite{NIPS2017_6211080f, 9681911, 9187874} require more computation or memory resources than the conventional weight averaging-based FL algorithms.
Motivated by this, this work aims to devise an efficient FL approach that simultaneously tackles data heterogeneity and communication resource limitations for FL in wireless networks.
Inspired by the success of centralized learning, different learning tasks often share the lower layers of neural networks responsible for feature extraction while the heterogeneity mainly focuses on the upper layers corresponding to complex pattern recognition\cite{bengio2013representation, lecun2015deep}.
We propose a novel FL framework that partially aggregates local model parameters of the devices in the learning process to learn a shared feature extractor, while the label predictor part are localized at devices for personalization.
This design effectively improves the learning performance of FL under heterogeneous local data distribution scenarios.
In addition, in view of the devices' limited wireless resources and energy budget, we devise a joint device scheduling, wireless bandwidth, and computation resources allocation scheme to improve the learning performance of FL in practical wireless networks.
The main contributions of this paper are summarized as follows:
\begin{itemize}
  \item To tackle the data heterogeneity across devices in the FL system, we devise a novel federated learning framework, namely partial model aggregation-FL (PMA-FL), in which devices only collaboratively train the lower layers of the neural networks while the upper layers are individually trained by each device for personalization. This design is able to reduce the data volumes in the transmission phase and improve the learning performance on heterogeneous local data distribution scenarios.

  \item To enable efficient FL in wireless networks, we minimize the global loss function while simultaneously considering devices' long-term energy budget, bandwidth limitation, and latency constraints. However, it is intractable to minimize the global loss due to its inexplicit form. To this end, we theoretically characterize the convergence bound of the considered FL system with the general non-convex loss function setting, finding a new metric, termed scheduled data sample volume, which is in an explicit form for the device scheduling policy. The minimum global loss function can be obtained by maximizing this metric.

  \item To maximize the scheduled data sample volume, we formulate a joint device scheduling, wireless bandwidth allocation, and computation-communication-time division optimization problem, which is a mixed-integer nonlinear programming problem and is challenging to solve. We first decouple the long-term stochastic problem into a deterministic one in each communication round with the assistance of the Lyapunov optimization framework. Then, we derive the optimal solution for time division policies through convex optimization techniques, develop a bisection method to address the optimal bandwidth allocation policy, and use the set expansion algorithm to achieve the device scheduling policy.

  \item Experiments show that the proposed FL algorithm achieves faster convergence speed and higher model accuracies compared with the benchmark schemes, improving 3.13\% and 11.8\% accuracy on MNIST and CIFAR-10 datasets, respectively.
      Moreover, the proposed joint device scheduling and resource management algorithm can reduce around 29\% energy budget or 20\% time budget and is able to achieve higher accuracies than the considered benchmarks on the MNIST dataset. On the CIFAR-10 dataset,  the proposed algorithm can obtain slightly higher accuracies than the benchmark schemes and reduce the 25\% energy budget or 12.5\% time budget.
\end{itemize}
\vspace{-0.6cm}
\subsection{Organization and Notations}
The rest of this paper is organized as follows: In Section \ref{sec:system_model}, we introduce the FL system and learning cost, then formulate the global loss minimization problem.
The convergence analysis and problem transformation are illustrated in Section \ref{sec:conver_ana}.
The joint device schedule, wireless bandwidth allocation, and time division algorithm are developed in \ref{sec:alg}.
Section \ref{sec:simulation} verifies the effectiveness of the proposed scheme by simulation. The conclusion is drawn in Section \ref{sec:conclusion}.
For convenience, we use ``$\buildrel \Delta \over = $'' to denote ``is defined to be equal to'', $\left|\cdot\right|$ denote the size operation of a set, $\nabla(\cdot)$ denote gradient operator, $\left\langle {\cdot,\cdot} \right\rangle $ denote inner product operator, and ``$\left\|  \cdot  \right\|$'' denote the $\ell_2$ norm throughout this paper. The main notations used in this paper are summarized in Table I.

\section{System Model}\label{sec:system_model}
After introducing the general FL system in this section, we will discuss FL with partial model aggregation (PMA), the computation cost, and the communication cost, and then formulate the problem.
\begin{figure}
  \begin{minipage}[t]{0.6\linewidth}
    \centering
    \setlength{\abovecaptionskip}{0cm}
    \includegraphics[scale=0.45]{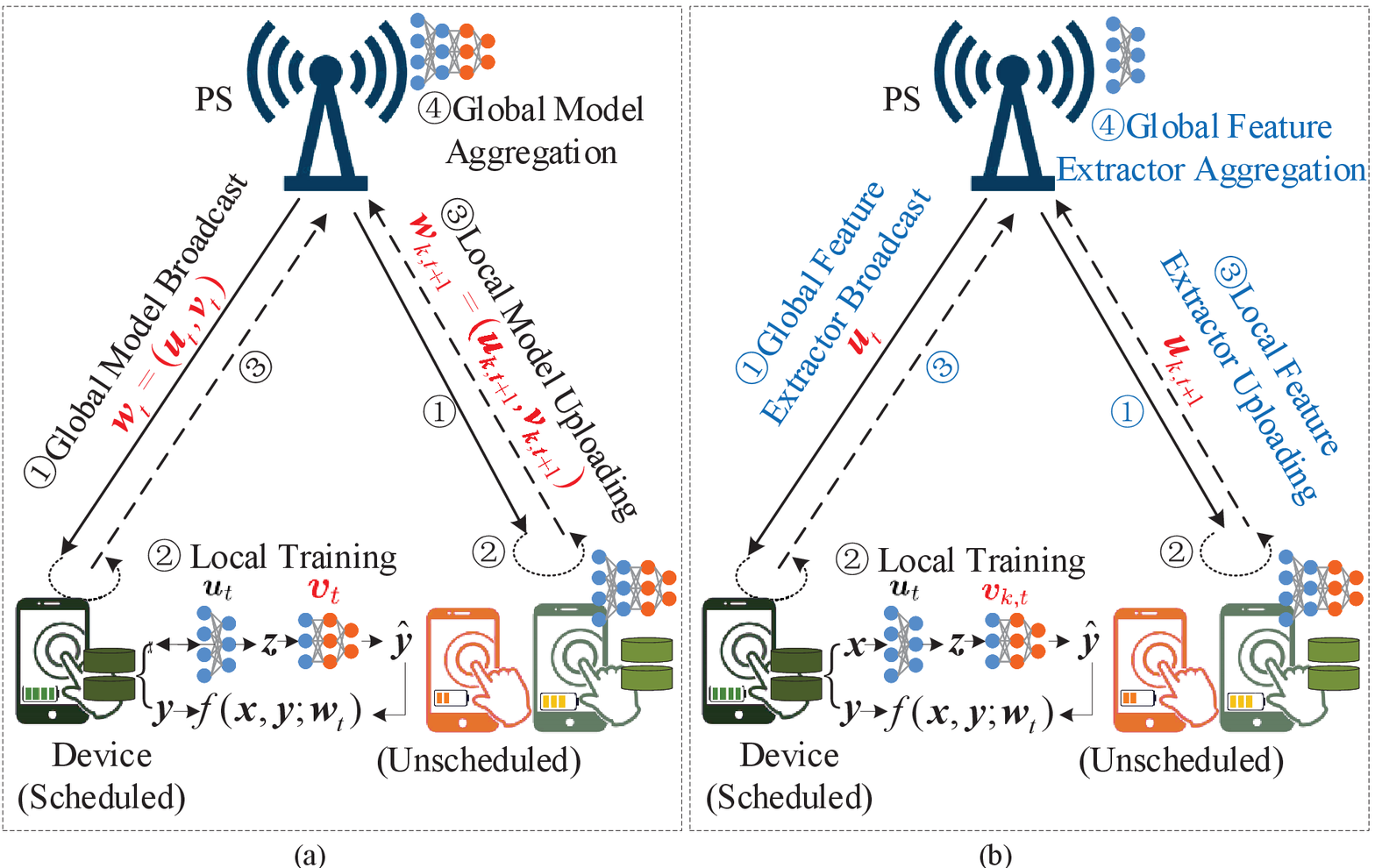}
    \caption{Illustrating the federated learning system and mechanism: (a) shows the traditional federated learning mechanism which trains a global model (including feature extractor and predictor); and (b) presents the federated learning mechanism with collaboratively train a feature extractor while the predictor is trained by each device itself for personalization.}
    \label{fig:sys_model}
  \end{minipage}
  \hspace{.1in}
  \begin{minipage}[t]{0.37\linewidth}
    \centering
    \setlength{\abovecaptionskip}{0cm}
    \includegraphics[scale=0.3]{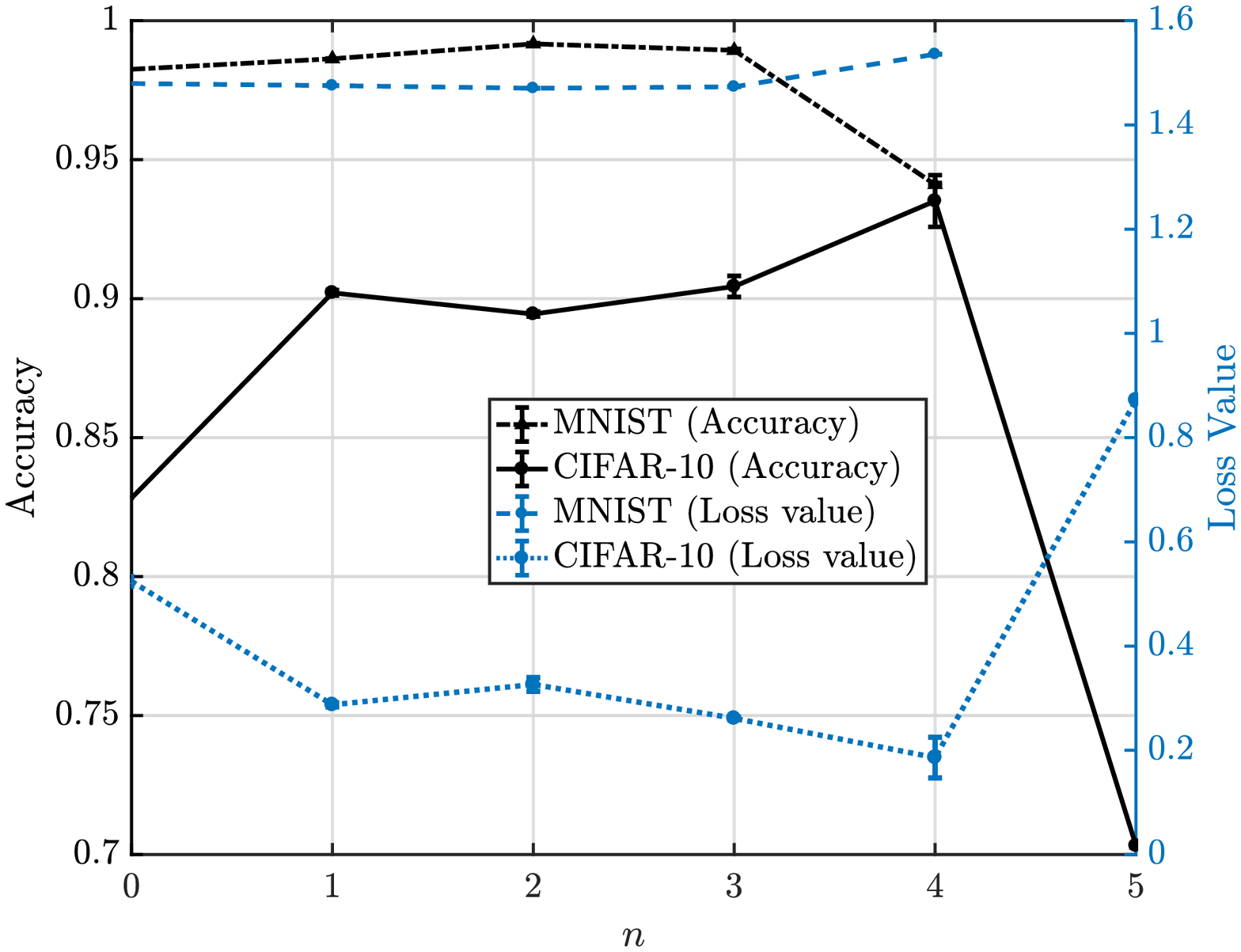}
    \caption{The test accuracy and global loss versus the number of layers of $\bm{u}_t$ (i.e., $n$): from without sharing parameters (each device solely train the model by using its own dataset) to sharing entire model parameters (FedAvg).}
    \label{fig:layerAcc}
  \end{minipage}%
  \vspace{-0.8cm}
\end{figure}



\begin{table}[ht]\tiny\label{tab:notation}\footnotesize
\caption{Notation Summary}
\begin{tabular}{p{1.8cm}|p{6.2cm}|p{1.8cm}|p{6.2cm}}
\hline
Notation & Definition & Notation & Definition \\
\hline
$\mathcal{K}$; $K$;   & Set of devices; size of $\mathcal{K}$
&$\mathcal{D}_k$; $D_k$ & Local dataset of device $k$; size of $\mathcal{D}_k$ \\
$F_k(\bm{u}_k,\bm{v}_k)$; $F(\bm{u},\bm{V})$ & Local loss function of device $k$; global loss function
&$\bm{w}_k$; $\bm{u}_k$; $\bm{v}_k$; $\bm{V}$ & Local model; local feature extractor; local predictor of device $k$; all devices' predictors \\
$\mathcal{D}$; $D$ & Overall dataset in the system; size of $\mathcal{D}$
&$\eta_u$; $\eta_v$ & Learning rate for feature extractor and predictor \\
$\bm{S}_t$; & Scheduling policy in round $t$, i.e., the set of scheduled devices
&$f_{k,t}$; $f_{k,\max}$ & CPU frequency of device $k$ in round $t$; maximum CPU frequency of device $k$\\
$p_{k,t}$; $p_{k,\max}$ & Transmit power of device $k$ in round $t$; maximum transmit power of device $k$
&$C_{k}$; $Q$ & Computation workload of one data sample at device $k$; Data size of feature extractor\\
$B$; $\bm{\theta}_t$ & Wireless transmission bandwidth;  the proportion of $B$ allocated to devices in round $t$
&$\bm{T}_t^{\rm{L}}$; $\bm{T}_t^{\text{U}}$ & Computation time and communication time for devices in round $t$\\
$E_k$ & Total energy budget of device $k$; & $T_{\max}$ & Maximum completion time for each round\\
\hline
\end{tabular}
\vspace{-0.8cm}
\end{table}

\vspace{-0.6cm}
\subsection{Federated Learning System}
In this work, we consider a typical FL setting for supervised learning, consisting of one PS and $K$ devices indexed by $\mathcal{K}= \left\{ {1,2, \cdots ,K} \right\}$, as shown in Fig. \ref{fig:sys_model}. Each device $k$ ($k \in \mathcal{K}$) has a local dataset $\mathcal{D}_k$ with $D_k =\left|\mathcal{D}_k\right|$ data samples. Without loss of generality, we assume there is no overlapping for datasets from different devices, i.e., $\mathcal{D}_k \cap \mathcal{D}_h = \emptyset ,(\forall k, h \in \mathcal{K})$. Thus, the whole dataset, $\mathcal{D} =  \cup \left\{ \mathcal{D}_k \right\}_{k = 1}^K$, is with total number of samples $D = \sum\nolimits_{k = 1}^K D_k$.

Given a data sample $(\bm{x}, \bm{y}) \in \mathcal{D}$, where $\bm{x} \in \mathbb{R}^d$ represents the input feature vector of the sample, and $\bm{y} \in \mathbb{R} $ is the corresponding ground-truth label. Let $\bm{z} \in \mathbb{R}^p $ be the latent feature space.
The machine learning model parameterized by $\bm{w}=[\bm{u}, \bm{v}]$ consists of two components: a feature extractor $\bm{x} \to \bm{z}$ parameterized by $\bm{u}$ and a predictor $\bm{z} \to \bm{\hat y}$ parameterized by $\bm{v}$.
Let $f(\bm{x},\bm{y};\bm{w})$ denotes the sample-wise loss function, which quantifies the error between the ground-truth label, $\bm{y}$, and the predicted output, $\bm{\hat y}$, based on model $\bm{w}$. Thus, the local loss function at device $k$, which measures the model error on its local dataset $\mathcal{D}_k$, is defined as
\begingroup
\setlength{\abovedisplayskip}{3pt}
\setlength{\belowdisplayskip}{3pt}
\begin{align}
F_k(\bm{w}_k) = F_k(\bm{u}_k,\bm{v}_k) \buildrel \Delta \over = \frac{1}{D_k}\sum\nolimits_{(\bm{x},\bm{y}) \in \mathcal{D}_k} {f(\bm{x}, \bm{y}; \bm{w}_k)},
\end{align}
\endgroup
where $\bm{w}_k$ denotes the model of device $k$; $\bm{u}_k$ and $\bm{v}_k$ correspond to the feature extractor and predictor, respectively.
Accordingly, the global loss function associated with all distributed local datasets is given by
\vspace{-0.8cm}
\begin{align}\label{eq:loss_fn}
F(\bm{w}_1, \cdots, \bm{w}_K) \buildrel \Delta \over = \frac{1}{D}\sum\nolimits_{k = 1}^K {D_k F_k(\bm{w}_k)}.
\end{align}
The federated learning process is done by solving the following problem
\begingroup
\setlength{\abovedisplayskip}{3pt}
\setlength{\belowdisplayskip}{3pt}
\begin{align}
\mathop {\min }\limits_{(\bm{w}_1, \cdots, \bm{w}_K)} \Big{(} F(\bm{w}_1, \cdots, \bm{w}_K) \buildrel \Delta \over = \frac{1}{D}\sum\nolimits_{k = 1}^K {D_k F_k(\bm{w}_k)} \Big{)}.
\end{align}
\endgroup
To preserve the data privacy of devices, the devices collaboratively learn $(\bm{w}_1, \cdots, \bm{w}_K)$ by only uploading local learning models $\bm{w}_k (k \in \mathcal{K})$ to the PS for periodical aggregation, instead of transmitting the raw training data.
\vspace{-0.6cm}
\subsection{Federated Learning with Partial Model Aggregation}
The main objective of the typical federated learning algorithms, such as the FedAvg \cite{pmlr-v54-mcmahan17a}, is to find an optimal shared global model $\bm{w}^* = \bm{w}_k^*$ $(\forall k \in \mathcal{K})$ that minimizes the global loss function $F(\bm{w}_1, \cdots, \bm{w}_K)$, as shown in Fig. \ref{fig:sys_model}(a).
However, the data distributions among different devices in real-world FL systems are often heterogeneous, namely statistical heterogeneity. In the presence of statistical data heterogeneity, the local optimal models may drift significantly from each other, and thus solely optimizing for the global model's accuracy leads to a poor generalization of each device.

Fortunately, the success of centralized learning in training multiple tasks or learning multiple classes simultaneously has shown that data often shares a global feature representation (i.e., $\bm{u}$), while the statistical heterogeneity across devices or tasks is mainly located at the labels' predictor (i.e., $\bm{v}$) \cite{bengio2013representation, lecun2015deep}. Thus, this work proposes PMA in the FL training process instead of aggregating the entire model, as shown in Fig. \ref{fig:sys_model}(b). Specifically, devices who participate in the FL training process only upload the parameters of feature extractor $\bm{u}$ for global aggregation and the predictor $\bm{v} $ is localized for personalization. The learning process repeats the following steps until the model converges. The combination of the steps is referred to as a global round.
\begin{itemize}
  \item \textbf{Device Selection}: The PS collects communication and computation information from each device and determines the set of scheduled devices in the current round, which is denoted by $\bm{S}_t$. Let $\alpha_{k,t} \in \{0, 1\}$ denotes the scheduling indicator of device $k$ in round $t$, where $\alpha_{k,t}\!=\!1$ indicates that device $k$ is scheduled in round $t$, $\alpha_{k,t}\!=\!0$ otherwise. Thus, we have $\bm{S}_t \!=\! \{k:\alpha_{k,t}=1, \forall k \in \mathcal{K}\}$.
  \item \textbf{Global Feature Extractor Broadcast}: In each round $t$, the PS broadcasts the current latest global feature extractor $\bm{u}_t$ to all scheduled devices.
  \item \textbf{Local Model Training}: All scheduled devices update their local models after receiving the global feature extractor $\bm{u}_t$. For device $k$, its local feature extractor in round ($t+1$) is updated as
          \vspace{-0.6cm}
          \begin{equation}\label{eq:feature_update}
          \setlength{\abovedisplayskip}{3pt}
          \setlength{\belowdisplayskip}{1pt}
          \bm{u}_{k,t + 1} = \bm{u}_{t} - \alpha_{k,t}{\eta_u}{\nabla_{\bm{u}}}F_k(\bm{u}_t,\bm{v}_{k,t}),
          \end{equation}
          and its predictor is updated by
           \vspace{-0.4cm}
          \begin{equation}\label{eq:predictor_update}
          \setlength{\abovedisplayskip}{3pt}
          \setlength{\belowdisplayskip}{3pt}
            \bm{v}_{k,t + 1} = \bm{v}_{k,t} - {\alpha_{k,t}}{\eta_v}{\nabla_{\bm{v}}}{F_k}(\bm{u}_t,\bm{v}_{k,t}),
          \end{equation}
      where $\eta_u$ and $\eta_v$ represent the learning rate of feature extractor $\bm{u}$ and predictor $\bm{v}$, respectively. For ease of presentation, we use $\bm{V}=(\bm{v}_1,\bm{v}_2, \cdots, \bm{v}_K)$ denotes all the devices' predictors throughout this paper.
  \item \textbf{Global Feature Extractor Aggregation}: After finishing the local training, all scheduled devices upload their updated local feature extractors to the PS through wireless channels for aggregation. Specifically, the PS computes the global shared feature extractor as follows:
      \vspace{-0.4cm}
      \begin{align}\label{eq:aggregation}
        \bm{u}_{t + 1} =\frac{\sum\nolimits_{k=1}^K \alpha_{k,t}D_k \bm{u}_{k,t+1}}{\sum\nolimits_{k = 1}^K \alpha_{k,t}D_k}.
      \end{align}
\end{itemize}

To better illustrate the benefits of partial sharing the model parameters, we provide an experiment on both MNIST and CIFAR-10 datasets in Fig. \ref{fig:layerAcc}, where the data distribution of each device is non-IID. Specifically, each device possesses at most two classes of data samples and participates in the training in each round. The MNIST dataset is trained by a 4-layers multi-layer perceptron (MLP) model, and the CIFAR-10 dataset is trained by a 5-layers convolutional neural network (CNN) model. The detailed configurations are shown in the experimental setting part in Section \ref{sec:simulation}.
Fig. \ref{fig:layerAcc} shows that the MLP with only sharing the first two layers ($n=2$) in the training process obtains the highest accuracy and the CNN achieves the highest accuracy by aggregating the first four layers ($n=4$). One interesting result is that for both MLP and CNN, the global trained model ($n=4$ for the MLP and $n=5$ for the CNN) is less accurate than the local models of devices ($n=0$ for both the MLP and CNN) trained by their local datasets. Thus, aggregating the feature extractor with sufficient feature extracting ability in the training process is an efficient method to obtain better performance in the non-IID data distribution scenarios, instead of aggregating the entire model or solely training models on devices' local datasets.
\vspace{-0.6cm}
\subsection{Computation Cost}
In each global round $t$, the selected devices will perform local training after receiving the global feature extractor, $\bm{u}_t$, then uploading the trained local feature extractor parameters, $\bm{u}_{k,t+1}$ ($\forall k \in \bm{S}_t$), to the PS for aggregation. Let $f_{k,t}$ denote the CPU frequency of device $k$. Employing dynamic voltage and frequency scaling techniques [11], device $k$ can control the energy consumption for computation by adjusting the CPU frequency. Denote $f_{k,\max}$ the maximum CPU frequency of device $k$.
For any given machine learning model, the number of floating-point operations (FLOPs) required to one data sample for gradient calculation can be estimated, denoted by $G$ \cite{goodfellow2016deep}. Let $\zeta_k$ denote the number of CPU cycles required to process one floating-point operation, which depends on the CPU. Thus, the computation workload of one data sample at device $k$ is represented by $C_k = \zeta_k G$.
Based on the real measurement result in \cite{miettinen2010energy}, energy consumption by devices is proportional to the square of their frequency. Thus, given the computation time restriction, $T_{k,t}^{\rm{L}}$, the most energy efficient CPU frequency is $f_{k,t} = \frac{\tau D_k C_k}{T_{k,t}^{\rm{L}}}$, where $\tau$ is the the number of local iterations.
The corresponding energy consumption of device $k$ to perform local training is
\vspace{-0.6cm}
\begingroup
\setlength{\abovedisplayskip}{1pt}
\setlength{\belowdisplayskip}{3pt}
\begin{align}\label{eq:local_energy}
E_{k,t}^{\rm{L}} = \kappa \tau{D_k}C_k f_{k,t}^2 = \frac{\kappa \tau^3 D_k^3 C_k^3}{(T_{k,t}^{\rm{L}})^2},
\end{align}
\endgroup
where $\kappa$ denotes the devices' energy coefficient that hinges on chip architecture. Since the CPU frequency of device $k$ is restricted by $f_{k,\max}$, the computation time should satisfy
\begin{equation}\label{eq:time_ltcons}
\setlength{\abovedisplayskip}{3pt}
\setlength{\belowdisplayskip}{3pt}
T_{k,t}^{\rm{L}} \ge \frac{\tau D_k C_k}{f_{k,\max}}.
\end{equation}

In the above discussion, we have ignored the global feature extractor aggregation cost, because the PS usually has strong computation capability with negligible aggregation delay.
\vspace{-0.6cm}
\subsection{Communication Cost}
The frequency-division multiple access (FDMA) technique is employed in the FL system with a total available bandwidth of $B$ Hz for devices to upload their local feature extractors $\bm{u}_{k,t}$.
Let $\theta_{k,t}$ ($0\le \theta_{k,t} \le 1$) represent the proportion of wireless channel bandwidth allocated to device $k$ in round $t$ and $p_{k,t}$ denote the uplink transmission power of device $k$ ($k \in \mathcal{K}$).
We assume that the channel gain, $h_{k,t}$, between device $k$ and the PS remains unchanged within one round but varies independently and identically over rounds. Consequently, the achievable uplink rates for device $k$ in round $t$ can be characterized by Shannon capacity, i.e., $r_{k,t} = \theta_{k,t}B\log(1 + \frac{p_{k,t} h_{k,t}}{{\theta_{k,t}B{N_0}}})$,  where $N_0$ is the power density of noise.
Denote $Q$ by the data size of feature extractor ($\bm{u}_{k,t},\forall k \in \mathcal{K}, \forall t$). Given the maximum communication time $T_{k,t}^{\rm{U}}$, the most energy efficient transmission method is $r_{k,t}=\frac{Q}{T_{k,t}^{\rm{U}}}$ \cite{7762913}. Thus, the transmit power is
\begin{equation}
\setlength{\abovedisplayskip}{1pt}
\setlength{\belowdisplayskip}{3pt}
{p_{k,t}} = \frac{\theta_{k,t} B N_0}{h_{k,t}} \Big{(}2^{\frac{Q}{\theta_{k,t} B T_{k,t}^{\rm{U}}}} - 1 \Big{)}.
\end{equation}
The corresponding energy consumption is $E_{k,t}^{\rm{U}} = {p_{k,t}}T_{k,t}^{\rm{U}}$. Thus, the total energy consumption of device $k$ in round $t$ for both computation and communication is $E_{k,t}=E_{k,t}^{\rm{L}}+E_{k,t}^{\rm{U}}$. Let $p_{k,\max}$ denote the maximum transmit power of device $k$, then $0 \le p_{k,t} \le p_{k,\max}$. Thus, the communication time for device $k$ uploading its local feature extractor should satisfy
\begin{equation}\label{eq:time_utcons}
\setlength{\abovedisplayskip}{3pt}
\setlength{\belowdisplayskip}{3pt}
T_{k,t}^{\rm{U}} \ge \frac{Q}{\theta_{k,t}B\log \Big{(} {1 + \frac{{{p_{k,\max}}{h_{k,t}}}}{{\theta_{k,t}B{N_0}}}} \Big{)}}.
\end{equation}

Similar to many existing works as in \cite{9237168, 9605599, 9579038}, we ignore the global feature extractor broadcasting cost and mainly focus on the performance bottleneck of the battery and communication-constrained edge devices because the PS usually supplied by the grid is energy-sufficient. Moreover, the broadcasting process occupies the entire bandwidth and the transmit power of the PS is usually large, the transmission delay is negligible.
\vspace{-0.6cm}
\subsection{Problem Formulation}
The objective of this work is to minimize the expected global loss $\mathbb{E}[F(\bm{u}_T, \bm{V}_T)]$ after $T$ rounds under the energy budget constraints of devices.
To this end, we jointly optimize the device scheduling, bandwidth allocation, computation time, and communication time allocation policy.
Denote $\bm{\theta}_t=(\theta_{1,t}, \theta_{2,t},\cdots, \theta_{K,t})$ as the proportions of the overall wireless bandwidth allocated to different devices in round $t$.
Let $\bm{T}_t^{\rm{L}}=(T_{1,t}^{\rm{L}}, T_{2,t}^{\rm{L}},\cdots, T_{K,t}^{\rm{L}})$ and $\bm{T}_t^{\rm{U}}=(T_{1,t}^{\rm{U}}, T_{2,t}^{\rm{U}},\cdots, T_{K,t}^{\rm{U}})$ denote the computation time and communication time for all devices in round $t$, respectively.
We formulate the problem as follows:
\begingroup
\setlength{\abovedisplayskip}{3pt}
\setlength{\belowdisplayskip}{3pt}
\begin{align}
\mathcal{P}:~~~~~~~&\min_{\left\{ \bm{S}_t,\bm{\theta}_t,\bm{T}_t^{\rm{L}},\bm{T}_t^{\text{U}},\right\}_{t = 0}^{T-1}}~ \mathbb{E}\left[F(\bm{u}_T, \bm{V}_T) \right]\label{prob:P}\\[-0.2cm]
\text{s.~t.~~}& (\text{\ref{eq:time_ltcons}}), (\text{\ref{eq:time_utcons}}),\label{cons:P_1}\tag{\theequation a}\\[-0.2cm]
& \sum\nolimits_{t = 0}^{T-1} {{E_{k,t}} \le } {E_k}, \forall k \in \mathcal{K}, \label{cons:P_2}\tag{\theequation b}\\[-0.2cm]
&\alpha_{k,t} \in \left\{ {0,1} \right\},\forall k \in \mathcal{K}, \forall t.\label{cons:P_3}\tag{\theequation c}\\[-0.2cm]
&\sum\nolimits_{k = 1}^K \theta_{k,t} \le 1, \forall t,\label{cons:P_4}\tag{\theequation d}\\[-0.2cm]
&0 \le \theta_{k,t} \le 1, \forall k \in \mathcal{K}, \forall t,\label{cons:P_5}\tag{\theequation e}\\[-0.2cm]
&T_{k,t}^{\rm{L}} + T_{k,t}^{\rm{U}} \le {T_{\max }}, \forall k \in \mathcal{K}, \forall t,\label{cons:P_6}\tag{\theequation f}
\end{align}
\endgroup
In problem $\mathcal{P}$,
(\ref{cons:P_1}) restricts the computation and communication time.
(\ref{cons:P_2}) indicates that for each device, the total energy consumption for both computation and communication over $T$ global rounds cannot exceed its given budget.
(\ref{cons:P_3}) indicates that which devices are scheduled in each round.
(\ref{cons:P_4}) assures that the wireless bandwidth resource allocated to all devices would not exceed the total available bandwidth resource.
(\ref{cons:P_5}) imposes restrictions on the wireless bandwidth resource allocated to each device.
(\ref{cons:P_6}) stipulates that the completion time for the participating devices in one round cannot exceed its maximum allowable delay $T_{\max}$.

Problem $\mathcal{P}$ involves a combinatorial optimization over the multi-dimensional discrete and continuous space, which is challenging to solve. Two major challenges of solving problem $\mathcal{P}$ are:
\begin{enumerate}[fullwidth,itemindent=1em,label=\arabic*)]
  \item \textbf{Inexplicit form of the objective function}: Since the evolutions of the feature extractor $\bm{u}_t$ and predictors $\bm{v}_{k,t}$ are complex in the training process, it is intractable to solve the close-form expression of $\mathbb{E}\left[F(\bm{u}_T, \bm{V}_T) \right]$.
  \item \textbf{Unknown future information}: The optimal solution of $\mathcal{P}$ requires exact channel state and devices' energy status information of all rounds at the beginning of training, which is impractical in real-world systems.
\end{enumerate}

To tackle these challenges, we first analyze the convergence bound of the considered PMA-FL algorithm and transform problem $\mathcal{P}$ into optimizing the convergence bound.

\section{Convergence Analysis and Problem Transformation}\label{sec:conver_ana}
In this section, we start with convergence analysis of the considered PMA-FL algorithm to find a metric, i.e., scheduled data size, which is in an explicit form with respect to the device schedule. Then, we transform problem $\mathcal{P}$ into maximizing this metric, so as to obtain the minimum global loss function when the FL converges. To address the challenge brought by the long-term energy constraint, we further transform the problem into a deterministic problem in each communication round by characterizing the Lyapunov drift-plus-penalty ratio function with the assistance of the Lyapunov optimization framework.
\vspace{-0.6cm}
\subsection{Convergence Anaysis}
We now investigate the convergence behaviour of the federated learning algorithm with partial model aggregation. To facilitate analysis, we make the following assumptions on the loss functions $F(\cdot)$.
\begin{ass}\label{assump:one}
(Lipschitz continuous): All loss function ${F_k}(\bm{u},\bm{v}_k)$ are continuously differentiable with respect to $\bm{u}$ and $\bm{v}_k$, and there exist constants $L_{\bm{u}}$, $L_{\bm{v}}$, $L_{\bm{uv}}$, and $L_{\bm{vu}}$ such that for each ${F_k}(\bm{u},\bm{v}_k)$ ($k \in \mathcal{K}$),
\begin{itemize}
\item $\nabla_{\bm{u}}{F_k}(\bm{u},{\bm{v}_k})$ is  $L_{\bm{u}}$-Lipschitz continuous with $\bm{u}$ and $L_{\bm{uv}}$-Lipschitz continuous with $\bm{v}_k$, that is,
    \begingroup
    \setlength{\abovedisplayskip}{-3pt}
    \setlength{\belowdisplayskip}{-3pt}
    \begin{align}
    \left\| {\nabla_{\bm{u}}{F_k}(\bm{u}, {\bm{v}_k}) \!-\! \nabla_{\bm{u}}{F_k}(\bm{u}', \bm{v}_k)} \right\| \le {L_{\bm{u}}}\left\|\bm{u} \!-\! \bm{u}'\right\|,
    \end{align}
    \endgroup
    and
    \vspace{-0.6cm}
    \begin{align}
    \left\| {\nabla_{\bm{u}}{F_k}(\bm{u},{\bm{v}_k}) \!-\! \nabla_{\bm{u}}{F_k}(\bm{u},\bm{v'}_k)} \right\| \!\le\! L_{\bm{uv}}\left\| \bm{v}_k \!-\! \bm{v'}_k \right\|.
    \end{align}
\item $\nabla_{\bm{v}}{F_k}(\bm{u},{\bm{v}_k})$ is  $L_{\bm{v}}$-Lipschitz continuous with $\bm{v}_k$ and $L_{\bm{vu}}$-Lipschitz continuous with $\bm{u}$.
\end{itemize}
\end{ass}

\begin{ass}\label{assump:two}
(Partial Gradient Diversity): There exist $\delta \ge 0$ and $\rho \ge 0$ such that for all $\bm{u}$ and $\bm{V}$, i.e.,
\begingroup
\setlength{\abovedisplayskip}{-5pt}
\setlength{\belowdisplayskip}{-3pt}
\begin{align}
{\left\| {\nabla_{\bm{u}}{F_k}(\bm{u}, \bm{v}_k)} \right\|^2} \le \delta^2  + \rho^2 {\left\| {\nabla_{\bm{u}}F(\bm{u}, \bm{V})} \right\|^2}.
\end{align}
\endgroup
\end{ass}

Assumption \ref{assump:one} is not stringent, which is satisfied by most deep neural networks. In fact, the convolution layer, linear layer, and some nonlinear activation functions (e.g., Sigmoid and tanh have already proved to be Lipschitz \cite{abbasnejad2018deep}. Based on \cite{abbasnejad2018deep}, a deep neural network defined by a composition of functions is a Lipschitz neural network if the functions in all layers are Lipschitz. Thus, most neural networks have the Lipschitz continuous gradients.
Assumption \ref{assump:two} is widely used in the convergence analysis in FL algorithms, e.g., \cite{9210812, sifaou2021robust}.
To begin with, we first derive a key lemma, proved in Appendix \ref{App:A1}, to assist our analysis as follows:

\begin{lem}\label{lem:lip_F}
Let Assumption $\rm{\ref{assump:one}}$ holds, we have
\vspace{-0.6cm}
\begin{multline}\label{eq:lemma1F_differ}
F\left( {\bm{u}_{t + 1},\bm{V}_{t + 1}} \right) - F(\bm{u}_t,\bm{V}_t)
\le \left\langle {\nabla_{\bm{u}}{F}(\bm{u}_t,\bm{V}_t),\bm{u}_{t + 1} \!-\! \bm{u}_t} \right\rangle  + \frac{1 + \chi}{2}{L_{\bm{u}}}{\left\| {\bm{u}_{t+1} \!-\! \bm{u}_t} \right\|^2}\\
+  \frac{1}{D}\sum\nolimits_{k = 1}^K {D_k} \Big{(} \left\langle {\nabla_{\bm{v}}{F_k}(\bm{u}_t,\bm{v}_{k,t}),\bm{v}_{k,t+1} - \bm{v}_{k,t}} \right\rangle
 + \frac{1 + \chi}{2}{L_{\bm{v}}}{\left\| {\bm{v}_{k,t+1} \!-\! \bm{v}_{k,t}} \right\|^2}\Big{)},
\end{multline}
where $\chi = \max \left\{ L_{\bm{uv}}, L_{\bm{vu}} \right\}/\sqrt{L_{\bm{u}} L_{\bm{v}}}$, which measures the relative cross-sensitivity of $\nabla_{\bm{u}}{F_k}(\bm{u},\bm{v}_k)$ with respect to $\bm{v}_k$ and $\nabla_{\bm{v}}{F_k}(\bm{u},\bm{v}_k)$ with respect to $\bm{u}$.
\end{lem}

Based on Lemma \ref{lem:lip_F}, we derive the one-round global loss reduction bound in Appendix \ref{App:A2}, which is summarized in the following Lemma.
\begin{lem}\label{thm:oneround_conver}
Let Assumption $\rm{\ref{assump:one}}$ and Assumption $\rm{\ref{assump:two}}$ hold. The learning rate satisfy $ \eta_{\bm{u}}\le\frac{1}{(\chi + 1) L_{\bm{u}}}$, $\eta_v \le \frac{2}{(\chi + 1)L_{\bm{v}}}$, we have
\begingroup
\setlength{\abovedisplayskip}{-5pt}
\setlength{\belowdisplayskip}{-5pt}
\begin{multline}\label{eq:theorem1}
\mathbb{E}\left[F\left( {\bm{u}_{t + 1},\bm{V}_{t + 1}} \right) - F(\bm{u}_t,\bm{V}_t)\right]
\le \frac{1}{2}{\eta _u}\left( {\frac{4}{D^2}\Big{(}D - \sum\nolimits_{k = 1}^K \alpha_{k,t} D_k \Big{)}^2{\rho^2} - 1} \right)\mathbb{E}{\left\| {{\nabla_{\bm{u}}}F(\bm{u}_t,\bm{V}_t)} \right\|^2}\\
+ 2{\eta_u}\frac{\Big{(}D - \sum\nolimits_{k = 1}^K \alpha_{k,t} D_k\Big{)}^2 \delta^2}{D^2}.
\end{multline}
\endgroup
\end{lem}

According to Lemma \ref{thm:oneround_conver}, one can find that the number of data samples scheduled in each round, i.e., $\sum\nolimits_{k = 1}^K {{\alpha_{k,t}}{D_k}}$, is the main contributor to the convergence rate of training.
Based on Lemma \ref{thm:oneround_conver}, we derive the convergence performance of the proposed PMA-FL algorithm after $T$ training rounds in the following theorem, proved in Appendix \ref{App:A3}.

\begin{thm}\label{lem:Tround_conver}
Let Assumption $\rm{\ref{assump:one}}$ and Assumption $\rm{\ref{assump:two}}$ hold. The learning rate satisfy $ \eta_{\bm{u}}\le\frac{1}{(\chi + 1) L_{\bm{u}}}$, $\eta_v \le \frac{2}{(\chi + 1)L_{\bm{v}}}$, the convergence bound in the $T$-th global round is given by
\vspace{-0.6cm}
\begin{multline}\label{eq:the1eq}
\mathbb{E}\left[F(\bm{u}_T, \bm{V}_T) - F(\bm{u}^*, \bm{V}^*)\right] \le \Big{(}\mathbb{E}\left[F(\bm{u}_0, \bm{V}_0) - F(\bm{u}^*, \bm{V}^*)\right]\Big{)}\mathop \prod \nolimits_{t = 0}^{T-1} {A_t}\\
+ \sum\nolimits_{t = 0}^{T-1} \frac{2\eta_u \delta^2}{D^2}\Big{(} {D - \sum\nolimits_{k = 1}^K \alpha_{k,t}D_k} \Big{)}^2 \mathop \prod \nolimits_{j = t + 1}^{T-1} {A_j},
\end{multline}
where $A_t=1 + \eta_u L_{\bm{u}}(\frac{4}{D^2}{\big{(}D - \sum\nolimits_{k = 1}^K \alpha_{k,t}D_k\big{)}^2} \rho^2 - 1 )$.
\end{thm}

From Theorem \ref{lem:Tround_conver}, we can conclude when $t$ trends to infinity with $0<\rho<\frac{1}{2}$ (i.e., $0<A_t<1$):
1) The FL training converges since $\prod \nolimits_{t = 0}^{T-1} {A_t}$ turns to 0 as $T$ increases, resulting in the first term in the right-hand side (RHS) of \eqref{eq:the1eq} converges to zero and the second term in the RHS of \eqref{eq:the1eq} approaches to be fixed.
2) A gap, i.e., the second term in the RHS of \eqref{eq:the1eq}, exists between $F(\bm{u}_T, \bm{V}_T)$ and $F(\bm{u}^*, \bm{V}^*)$.
Particularly, $A_t$ and the second term in the RHS of \eqref{eq:the1eq} affect the convergence speed and learning accuracy, respectively. A small $A_t$ induces a fast learning speed, and a small $\sum\nolimits_{t = 0}^{T-1} \frac{2\eta_u \delta^2}{D^2}(D - \sum\nolimits_{k = 1}^K \alpha_{k,t}D_k)^2 \mathop \prod \nolimits_{j = t + 1}^{T-1} {A_j}$ results in a small loss function and high learning accuracy.
Increasing $\sum\nolimits_{k = 1}^K \alpha _{k,t}D_k$ in each round helps $\prod \nolimits_{t = 0}^{T-1} {A_t}$ approach 0 faster and decreases the second term in the RHS of \eqref{eq:the1eq}.
These observations motivate us to maximize $\sum\nolimits_{k = 1}^K \alpha _{k,t}D_k$ in each round to improve the learning performance of PMA-FL.
Note that, Theorem \ref{lem:Tround_conver} reveals the impact of unbalanced data on the convergence performance of PMA-FL and builds the bridge between the scheduled data samples maximization and the global loss minimization from a theoretical perspective.

\vspace{-0.6cm}
\subsection{Problem Transformation}
Motivated by Theorem \ref{lem:Tround_conver}, we maximize the overall scheduled data size, i.e., $\sum\nolimits_{t = 0}^{T - 1} \sum\nolimits_{k = 1}^K {\alpha_{k,t}}D_k$, for the global loss function minimization. Thus, we transform problem $\mathcal{P}$ into the following one.
\begingroup
\setlength{\abovedisplayskip}{-0.08cm}
\setlength{\belowdisplayskip}{-0.15cm}
\begin{align}
\mathcal{P}_1:~\max_{\left\{\bm{S}_t,\bm{T}_t^{\rm{L}},\bm{T}_t^{\text{U}},\bm{\theta}_t \right\}_{t = 0}^{T-1}}~&\sum\nolimits_{t = 0}^{T - 1} \sum\nolimits_{k = 1}^K {\alpha_{k,t}}{D_k} \label{prob:P1}\\[-0.2cm]
\text{s.~t.~~~~}& (\text{\ref{eq:time_ltcons}}), (\text{\ref{eq:time_utcons}}),(\text{\ref{cons:P_2}}),(\text{\ref{cons:P_3}}),(\text{\ref{cons:P_4}}),(\text{\ref{cons:P_5}}),(\text{\ref{cons:P_6}}). \notag
\end{align}
\endgroup

Problem $\mathcal{P}_1$ is difficult to solve due to the long-term energy constraint and the unknown future information about channel condition for devices. To enable online dynamic scheduling for devices, we construct a virtual queue $q_k(t)$ for each device $k$ to indicate the gap between the cumulative energy consumption till round $t$ and the budget, evolving according to
\begingroup
\setlength{\abovedisplayskip}{0pt}
\setlength{\belowdisplayskip}{0pt}
\begin{align}
q_k(t + 1) = \max\left\{q_k(t) + \alpha_{k,t} E_{k,t} - \frac{E_k}{T},0 \right\},
\end{align}
\endgroup
with an initial value ${q_k}(0) = 0$ for all devices.
Inspired by the drift-plus-penalty algorithm of Lyapunov optimization \cite{neely2010stochastic}, the online scheduling aims to solve the following problem,
\begingroup
\setlength{\abovedisplayskip}{-2pt}
\setlength{\belowdisplayskip}{-3pt}
\begin{align}
\mathcal{P}_2:\min_{\left\{ \bm{S}_t,\bm{T}_t^{\rm{L}},\bm{T}_t^{\text{U}},\bm{\theta}_t \right\}}&-V\sum\nolimits_{k=1}^K {\alpha_{k,t}}{D_k} + \sum\nolimits_{k = 1}^K {q_k}(t){\alpha_{k,t}}E_{k,t} \label{prob:P2}\\[-0.2cm]
\text{s.~t.~~~~}& (\text{\ref{eq:time_ltcons}}), (\text{\ref{eq:time_utcons}}),(\text{\ref{cons:P_3}}),(\text{\ref{cons:P_4}}),(\text{\ref{cons:P_5}}),(\text{\ref{cons:P_6}}). \notag
\end{align}
\endgroup
where $V \ge 0$ is an adjustable weight parameter to balance scheduled data size and energy consumption. A large $V$ indicates that the optimization objective emphasizes more on the scheduled data size for improving the learning performance and less on energy consumption minimization, and vice versa.

\section{Energy-Efficient Dynamic Device Scheduling and Resource Management}\label{sec:alg}
In this section, we solve the deterministic combinatorial problems $\mathcal{P}_2$ in each communication round. We first exploit the dependences among $\bm{S}_t$, $\bm{\theta}_t$, $\bm{T}_t^{\rm{L}}$, and $\bm{T}_t^{\rm{U}}$ in problem $\mathcal{P}_2$ and transform it into an equivalent problem that joint optimizing $\bm{S}_t$, $\bm{\theta}_t$, and $\bm{T}_t^{\rm{L}}$. Then we decompose it into three sub-problems and deploy an alternative optimization technique to obtain its optimal solution.
For the convenience of analysis, we rewrite the local feature extractor uploading energy consumption as
\vspace{-0.5cm}
\begingroup
\setlength{\abovedisplayskip}{1pt}
\setlength{\belowdisplayskip}{1pt}
\begin{align}\label{eq:trans_energy}
E_{k,t}^{\rm{U}} = {p_{k,t}}T_{k,t}^{\rm{U}} = \frac{\theta_{k,t} B N_0 T_{k,t}^{\rm{U}}}{h_{k,t}}\Big{(} 2^{\frac{Q}{\theta_{k,t}B T_{k,t}^{\rm{U}}}} - 1 \Big{)},
\end{align}
\endgroup
which is a non-increasing function with respect to $T_{k,t}^{\rm{U}}$. Thus, by taking into account the constraint (\ref{cons:P_6}), the optimal communication time satisfies $T_{k,t}^{\rm{U}} = {T_{\max }} - T_{k,t}^{\rm{L}}$. Based on this, we can simplify problem $\mathcal{P}_2$ as the following equivalent problem,
\vspace{-0.4cm}
\begin{align}
\mathcal{P}_3:\min_{\left\{ \bm{S}_t,\bm{\theta}_t, \bm{T}_t^{\rm{L}} \right\}}&-V\sum\nolimits_{k=1}^K {\alpha_{k,t}}{D_k} + \sum\nolimits_{k = 1}^K {q_k}(t){\alpha_{k,t}}E_{k,t} \label{prob:P3}\\[-0.2cm]
\text{s.~t.~~~~}& (\text{\ref{cons:P_3}}), (\text{\ref{cons:P_4}}),(\text{\ref{cons:P_5}}), \notag\\[-0.2cm]
&\tau D_k C_k / f_{k,\max} \le T_{k,t}^{\rm{L}} \le {T_{\max}} - Q/r_{k,t}^{\max}(\theta_{k,t}),\label{cons:P3_1}\tag{\theequation a}
\end{align}
where
\vspace{-0.6cm}
\begingroup
\setlength{\abovedisplayskip}{1pt}
\setlength{\belowdisplayskip}{1pt}
\begin{align}
r_{k,t}^{\max}(\theta_{k,t}) = \theta_{k,t} B \log \Big{(}1 + \frac{p_{k,\max} h_{k,t}}{\theta_{k,t} B N_0} \Big{)}.
\end{align}
\endgroup
However, problem $\mathcal{P}_3$ is a mixed integer non-linear programming problem, which is still difficult to solve. In the below, we decompose it into three sub-problems and solve them one by one.
\vspace{-0.6cm}
\subsection{Local Training Time Allocation}
For any given device scheduling policy $\bm{S}_t$ and bandwidth allocation strategy $\bm{\theta}_t$, we can decompose the computation time allocation problem as follows,
\begingroup
\setlength{\abovedisplayskip}{-0.15cm}
\setlength{\belowdisplayskip}{-0.18cm}
\begin{align}
\mathcal{P}_4:\min_{\bm{T}_t^{\rm{L}}}~~~&~~\sum\nolimits_{k \in \bm{S}_t} q_k(t)E_{k,t} \label{prob:P4}\\[-0.2cm]
\text{s.~t.~~~~}& (\text{\ref{cons:P3_1}}). \notag
\end{align}
\endgroup
We can prove that $\mathcal{P}_4$ is convex, and obtain its optimal solution as summarized in Lemma \ref{lem:time_convex}, proved in Appendix \ref{App:A4}.
\begin{lem}\label{lem:time_convex}
Problem $\mathcal{P}_4$ is a convex problem and its optimal solution is given as
\vspace{-0.4cm}
\begin{align}\label{eq:optimaltime}
T_{k,t}^{\rm{L},*} = \left\{ {\begin{array}{*{20}{c}}
{\frac{\tau D_k C_k}{f_{k,\max}},}&{T_{k,t}^{\rm{L},0} \le \frac{\tau D_k C_k}{f_{k,\max}},}\\
{{T_{\max }} - \frac{Q}{{r_{k,t}^{\max }(\theta_{k,t})}},}&{T_{k,t}^{\rm{L},0} \ge {T_{\max }} - \frac{Q}{{r_{k,t}^{\max }(\theta_{k,t})}},}\\
{T_{k,t}^{\rm{L},0}},&{\rm{otherwise}.}
\end{array}} \right.
\end{align}
where $T_{k,t}^{\rm{L},0}$ satisfies the equality $\frac{\partial E_{k,t}}{\partial T_{k,t}^{\rm{L},0}} = 0$.
\end{lem}
In fact, constraint (\ref{cons:P3_1}) imposes restrictions on the maximum frequency and transmit power and is usually inactive in practical system design because this usually can be satisfied by modifying the minimum required latency constraint, $T_{\max}$, and bandwidth $B$. Thus, we have the following remark.
\begin{rem}
In general, the optimal computation time satisfy $T_{k,t}^{\rm{L},*}= T_{k,t}^{\rm{L},0}$, which is equivalent to $\frac{\partial E_{k,t}^{\rm{L}}}{\partial T_{k,t}^{\rm{L}}} = \frac{\partial E_{k,t}^{\rm{U}}}{\partial T_k^{\rm{U}}}$. In other words, the computation time allocation policy is optimal when the power of local training equals that of wireless communication.
\end{rem}

\vspace{-0.6cm}
\subsection{Wireless Bandwidth Allocation}
For ease of presentation, we define an auxiliary function for each device $k \in \mathcal{K}$ as follows:
\begingroup
\setlength{\abovedisplayskip}{3pt}
\setlength{\belowdisplayskip}{3pt}
\begin{align}
g_k(\theta_{k,t}) = \exp \bigg{(} \frac{Q\ln 2}{\theta_{k,t}B(T_{\max} - T_{k,t}^{\rm{L}})} \bigg{)} - 1.
\end{align}
\endgroup
For any given computation time allocation decision $\bm{T}_t^{\rm{L}}$ and device scheduling policy $\bm{S}_t$, the wireless bandwidth allocation problem can be separated as,
\vspace{-0.4cm}
\begin{align}
\mathcal{P}_5:\min_{\bm{\theta}_t}&~~~~~~~h(\bm{\theta}_t) \label{prob:P5}\\[-0.2cm]
\text{s.~t.~~~~}& (\text{\ref{cons:P_4}}),(\text{\ref{cons:P_5}}),\notag\\[-0.2cm]
& Q/(T_{\max} - T_{k,t}^{\rm{L}}) \le r_{k,t}^{\max}(\theta_{k,t}), \label{cons:P5_1}\tag{\theequation a}
\end{align}
where
\vspace{-0.6cm}
\begingroup
\setlength{\abovedisplayskip}{1pt}
\setlength{\belowdisplayskip}{1pt}
\begin{align}
h(\bm{\theta}_t) = \sum\nolimits_{k \in \bm{S}_t} \theta_{k,t} \frac{{N_0 B q_k(t)(T_{\max} - T_{k,t}^{\rm{L}})}}{{{h_{k,t}}}}{g_k}(\theta_{k,t}).
\end{align}
\endgroup

Problem $\mathcal{P}_5$ is a standard convex optimization problem, its proof is similar to that for Lemma \ref{lem:time_convex} and thus omitted for brevity. Applying Karush-Kuhn-Tucker condition \cite{boyd2004convex}, the optimal solution for $\bm{\theta}_t$ satisfies
\vspace{-0.4cm}
\begingroup
\setlength{\abovedisplayskip}{1pt}
\setlength{\belowdisplayskip}{1pt}
\begin{align}\label{eq:opt_B_cond}
\frac{\partial h(\bm{\theta}_t)}{\partial \theta_{k,t}} =  - \lambda^*, \forall k \in \bm{S}_t,
\end{align}
\endgroup
where $\lambda^*$ is the optimal Lagrange multiply and $\sum\nolimits_{k \in \bm{S}_t} \theta_{k,t} = 1$. Thus, for each device $k$, we have
\vspace{-0.8cm}
\begin{align}
{g_k}(\theta_{k,t}) + \theta_{k,t} g_k'(\theta_{k,t}) = \frac{{ - \lambda^* {h_{k,t}}}}{{{q_k}(t){N_0}B(T_{\max} - T_{k,t}^{\rm{L}})}},
\end{align}
its inverse function is
\vspace{-0.6cm}
\begin{align}\label{eq:theta_op}
\theta_{k,t}(\lambda^*) \!=\! \frac{{Q\ln 2}}{B(T_{\max} \!-\! T_{k,t}^{\rm{L}})\left(\mathcal{W}(\frac{\lambda^* h_{k,t}}{q_k(t) N_0 B (T_{\max} \!-\! T_{k,t}^{\rm{L}})e} \!-\! \frac{1}{e}) \!+\! 1 \right)},
\end{align}
where $\mathcal{W}$ refers to the principal branch of the Lambert $\mathcal{W}$ function, defined as the solution for $\mathcal{W}(x)e^{\mathcal{W}(x)}=x$, in which $e$ refers to the Euler's number.

In (\ref{eq:theta_op}), there still exists an unknown variable $\lambda^*$. The value of $\lambda^*$ is determined by the equation $\sum \nolimits_{k = 1}^K \theta_{k,t}(\lambda^*) = 1$. Since the expression of $\theta_{k,t}(\lambda^*) $ is complicated, it is difficult to solve the optimal $\lambda^*$. Below we propose a bisection search method to solve $\sum \nolimits_{k = 1}^K \theta_{k,t}(\lambda^*) = 1$. To proceed, we have the following Proposition.
\begin{prop}\label{prop:increase}
$\theta_{k,t}(\lambda)$ is a monotonically decreasing function with respect to $\lambda$.
\end{prop}
\begin{proof}
Since the Lagrange multiply $\lambda >0$, we have $\frac{\lambda h_{k,t}}{e{q_k}(t)N_0 B(T_{\max} - T_{k,t}^{\rm{L}})} - \frac{1}{e} >  - \frac{1}{e}$. Moreover, $\mathcal{W}(x)$ is a monotonically increasing function when $x \ge -\frac{1}{e}$. Thus, $\theta_{k,t}(\lambda)$ is a monotonically decreasing function with respect to $\lambda$.
\end{proof}

Based on Proposition \ref{prop:increase}, the bisection search method is employed to solve the equation. In the following, we derive the bisection search upper and lower bound on $\lambda$. Since $\lambda>0$, the lower bound of $\lambda$ is $\lambda_{\text{LB}}=0$. For deriving the upper bound, we have $\max_{k \in \bm{S}_t}\{\theta_{k,t}(\lambda)\} \ge 1/{|\bm{S}_t|}$, thus
\begingroup
\setlength{\abovedisplayskip}{-3pt}
\setlength{\belowdisplayskip}{1pt}
\begin{align}
\mathcal{W}(\frac{\lambda h_{k,t}}{{q_k}(t){N_0}B(T_{\max} - T_{k,t}^{\rm{L}})e} - \frac{1}{e}) \le \frac{|\bm{S}_t| Q\ln 2}{B(T_{\max} - T_{k,t}^{\rm{L}})} - 1.
\end{align}
\endgroup
Let ${\varphi _k} = \frac{|\bm{S}_t|Q\ln 2}{B(T_{\max} - T_{k,t}^{\rm{L}})}$, from the definition of Lambert $\mathcal{W}$ function, we have
\begingroup
\setlength{\abovedisplayskip}{-2pt}
\setlength{\belowdisplayskip}{1pt}
\begin{align}\label{eq:up_lag_bound}
\lambda_{\text{UB}} = \max_{k \in \bm{S}_t} \left\{ \frac{q_k(t){N_0}B(T_{\max} \!-\! T_{k,t}^{\rm{L}})\left((\varphi_k \!-\! 1) e^{\varphi_k} \!+\! 1 \right)} {h_{k,t}} \right\}.
\end{align}
\endgroup
According to the lower bound $\lambda_{\text{LB}}$ and upper bound $\lambda_{\text{UB}}$, the optimal Lagrange multiply, $\lambda^*$, can be solved by using the bisection search method. Furthermore, the optimal wireless bandwidth allocation policy $\bm{\theta}_t$ can be derived from (\ref{eq:theta_op}).
Based on the above analysis, we have the following remark.
\begin{rem}
From (\ref{eq:opt_B_cond}),
when the bandwidth allocation policy is optimal, all devices' energy consumption-bandwidth rates (i.e., $\frac{\partial h(\bm{\theta}_t)}{\partial \theta_{k,t}}$) are equal. This actual achieves the energy consumption balance between devices.
Moreover, similar to the proof of Proposition \ref{prop:increase}, it can be proved that the optimal bandwidth form in (\ref{eq:theta_op}) is monotonically decreasing with $h_{k,t}$ and increasing with $q_k(t)$. Thus, more bandwidth should be allocated to the devices with weaker channels (smaller $h_{k,t}$) and less remaining energy budgets (larger $q_k(t)$).
\end{rem}
\vspace{-0.6cm}
\subsection{Device Scheduling Policy}
Until now, for any given $\bm{S}_t$, the computation time allocation or wireless bandwidth allocation policies can be solved if one of them is fixed. Below we solve the joint computation time and wireless bandwidth allocation policy.
For clarity, we formulate the joint computation time allocation and bandwidth allocation problem under given device scheduling decision $\bm{S}_t$ as follows:
\begingroup
\setlength{\abovedisplayskip}{-0.2cm}
\setlength{\belowdisplayskip}{-0.15cm}
\begin{align}
\mathcal{P}_6:\min_{\left\{\bm{\theta}_t, \bm{T}_t^{\rm{L}} \right\}}& \sum\nolimits_{k \in \bm{S}_t} {q_k}(t)E_{k,t} \label{prob:P6}\\[-0.2cm]
\text{s.~t.~~~~}& (\text{\ref{cons:P_4}}),(\text{\ref{cons:P_5}}), (\text{\ref{cons:P3_1}}), \notag
\end{align}
\endgroup
which is a combination problem of $\mathcal{P}_4$ and $\mathcal{P}_5$. Building on the preceding results, the computation time allocation problem, $\mathcal{P}_4$, and the bandwidth allocation problem, $\mathcal{P}_5$, are both convex optimization problems, problem $\mathcal{P}_6$ is also a convex optimization problem.
Thus, we solve the joint computation time and wireless bandwidth allocation policies via iterations \cite{waldspurger2015phase} between problem $\mathcal{P}_4$ and problem $\mathcal{P}_5$.
Each iteration consists of two steps: (1) solving the optimal solution of problem $\mathcal{P}_5$ for given $\bm{T}_{t}^{\rm{L}}$; (2) solving the computation time allocation policy $\bm{T}_{t}^{\rm{L}}$  based on the obtained bandwidth allocation solution $\bm{\theta}_t$. The two steps are iterated until convergence.
For clarity, we summarize the detailed steps on joint optimization of computation time and wireless bandwidth in Algorithm 1. Based on the complexity analysis results in \cite{waldspurger2015phase}, the time complexity of Algorithm 1 is $\mathcal{O}(2K^{3.5})$.
\begin{algorithm}
\algsetup{linenosize=} \small
\caption{Computation time and Bandwidth Allocation}
\begin{spacing}{1.3}
\begin{algorithmic}[1]
\STATE Initialize $\bm{S}_t$, the computation time as $\widetilde{T}_t^{\rm{L}}$, and bandwith allocation policy $\bm{\theta}_t$, the tolerant error $\Upsilon >0$
\STATE Calculate the objective function value (\ref{prob:P2}), denote as $\mathcal{B}_0$
\REPEAT
    \STATE Calculate the upper bound of the Lagrange multiply $\lambda_{\text{UB}}$ based on (\ref{eq:up_lag_bound}), and let $\lambda_{\text{LB}}=0$
    \STATE Utilize the bisection search method to solve the optimal bandwidth allocation policy $\bm{\theta}_t$
    \STATE Solve the computation time allocation policy based on the obtained $\bm{\theta}_t$ by using (\ref{eq:optimaltime}), update $\widetilde{T}_t^{\rm{L}}$
    \STATE Calculate the objective function value (\ref{prob:P6}) of $\bm{S}_t$ by substituting the obtained $\widetilde{T}_t^{\rm{L}}$ and $\bm{\theta}_t$, denote as $\mathcal{B}_1$
    \STATE $\Delta = \mathcal{B}_0- \mathcal{B}_1$, update $\mathcal{B}_0=\mathcal{B}_1$
\UNTIL{$\Delta \le \Upsilon$}
\RETURN The computation time allocation policy $\widetilde{T}_t^{\rm{L}}$ and bandwith allocation policy $\bm{\theta}_t$
\end{algorithmic}
\end{spacing}
\end{algorithm}

Through the above analysis, we can solve the optimal value of the objective function in (\ref{prob:P3}) for any given device scheduling decision $\bm{S}_t$. An intuitive method to solve the optimal device scheduling solution is to solve the objective function value of all the possible device scheduling decisions first and then select the one with the minimum objective function value. However, this method has exponential time complexity $\mathcal{O}( K^{3.5} \times 2^{K+1})$ since there are total $\sum_{n=0}^K C_K^n = 2^K$ possible device scheduling decisions. To tackle this challenge, we have the following designs.

According to the objective function (\ref{prob:P3}), it is desirable to select devices with small $q_k(t)$ and $E_{k,t}$. The small $E_{k,t}$ can be achieved by strong channels or/and high computation efficiencies. To identify such devices, we first perform equal bandwidth allocation over all devices and then evaluate the resulting energy consumption of each device $\bar{E}_{k,t}$.
Specifically, each device $k$ is allocated the same portion, $\theta_{k,t}=\frac{1}{K}$, of the total bandwidth $B$, and then solve problem $\mathcal{P}_4$ to obtain the computation time allocation policy $\bm{T}_t^{\rm{L}}$. Then, by substituting $\theta_{k,t}=\frac{1}{K}$ and $\bm{T}_t^{\rm{L}}$ into the (\ref{eq:local_energy}) and (\ref{eq:trans_energy}), the estimated energy consumption is calculated as $\bar{E}_{k,t}=E_{k,t}^{\rm{U}}+ E_{k,t}^{\rm{L}}$.

Based on the evaluated energy consumption $\bar{E}_{k,t}$, we sort $\mathcal{C}_{k,t} = q_k (t) \bar{E}_{k,t}$ in the ascending order, and then use the set expansion algorithm \cite{9237168} to solve the device selection policy by incrementally adds devices into the selection set, $S$. Firstly, the devices with $q_k (t)=0$ are all added into $S$, denote this device set by $S_0$. Next, the devices with $q_k (t)>0$ are added into $S$ one by one in the ascending order of $\mathcal{C}_{k,t}$.
For each possible device scheduling set $S$, we perform Algorithm 1 to obtain the computation time and wireless bandwidth allocation decisions. Let $\mathcal{R}^*(S)=(\theta^*(S), T^*(S))$  denote the time and wireless bandwidth decision and $\mathcal{Y}(S)$ represent the corresponding objective function value of $S$, respectively. Denote $\mathcal{H}$ as the set of all possible device scheduling set $S$.

Note that, $\mathcal{Y}(S_0)=-V\sum\nolimits_{k \in S_0} D_k$ due to $q_k (t)=0$ ($\forall k \in S_0$). Since the energy consumption of users in $S_0$ does not affect the objective function value, the minimum required bandwidth should be allocated to them for saving more bandwidth resources for other users in $(S-S_0)$.
Moreover, we add the users with $q_k (t)>0$ one by one into $S$ and solve the $\mathcal{R}^*(S)$ and $\mathcal{Y}(S)$. For $S$, if its optimal computation time and wireless bandwidth allocation policy results in $-V D_k + q_k(t) E_{k,t}>0$ for the last added device $k$, we stop adding devices into $S$ and remove the last added device. Then, we obtain the device scheduling policy through comparing the objective function value of all $S \in \mathcal{H}$, i.e., $S_t^* = \argmin_{S \in \mathcal{H}} \mathcal{Y}(S)$. The computation time and optimal bandwidth allocation policy correspond to $T_t^*(S)$ and $\theta_t^*(S)$. For clarity, we summarize the detail steps of device scheduling in Algorithm 2, which obtains the device scheduling solution of problem $\mathcal{P}_1$ by solving at most $K$ times convex problem $\mathcal{P}_6$ and has polynomial time complexity $\mathcal{O}(2K^{4.5})$ , which is smaller than $O(K^{3.5}\times 2^{K+1})$ when $K > 1$.
\begin{algorithm}
\algsetup{linenosize=} \small
\caption{Device scheduling}
\begin{spacing}{1.3}
\begin{algorithmic}[1]
\STATE Input the virtual queue length $q_k(t)$ ($k \in \mathcal{K}$), initialize $V$
\STATE Sort $\mathcal{C}_{k,t}$ in ascending order.
\STATE Set $S_0=\{k:q_k(t)=0\}$, $S=S_0$ and $\mathcal{H}=\{S_0\}$
\FOR{$k=|S_0|+1, \cdots, K$}
    \STATE Update $S=S \cup \{k\}$
    \STATE Solve the optimal computation time and bandwidth allocation policy by Algorithm 1, i.e., $\mathcal{R}(S)=(\bm{T}_t^{\rm{L}}, \bm{\theta}_t)$.
    \IF{$-VD_k + q_k(t) E_{k,t}>0$}
        \STATE Break the circulation
    \ELSE
        \STATE Add $S$ into $\mathcal{H}$, i.e., $\mathcal{H}=\mathcal{H} \cup S$
    \ENDIF
\ENDFOR
\STATE Find the optimal device scheduling set $\bm{S}_t^* = \argmin_{S \in \mathcal{H}}\mathcal{Y}(S)$
\RETURN The optimal device scheduling set $\bm{S}_t^*$, computation time $\bm{T}_t^{\rm{L}}$ and wireless bandwidth allocation $\bm{\theta}_t$
\end{algorithmic}
\end{spacing}
\end{algorithm}

\vspace{-0.6cm}
\section{Numerical Results}\label{sec:simulation}
In this section, we evaluate the performance of the proposed energy-efficient dynamic device scheduling FL algorithm. In the simulation, all the codes are implemented in python 3.8 and Pytorch, running on a Linux server. We first present the evaluation setup and then show experimental results.
\vspace{-0.6cm}
\subsection{Experimental Setting}
The default experiment settings are given as follows unless specified otherwise.
\begin{enumerate}[fullwidth,itemindent=1em,label=\arabic*)]
\item Datasets and Models: We evaluate the proposed algorithm for an image classification task using MNIST and CIFAR-10 datasets.
    The MNIST dataset consists of 60,000 and 10,000 grey-valued digital images for training and test, respectively. Each image is a handwritten digital between 0 and 9 displayed as a 28$\times$28 pixel matrix.
    The CIFAR-10 dataset consists of 60000 32$\times$32 colour images in 10 classes, with 50000 training images and 10000 test images. For both MNIST and CIFAR-10, we first classify the training data samples according to their labels, then randomly split each class of data samples into $ 2K/10$ shards, finally randomly distribute two shards of data samples to each device.
    For the MNIST dataset, we train a MLP, which consists of 4 layers with 550346 parameters in total. The first four layers have 784, 512, 256, and 64 units, respectively. Each of these layers is activated by the ReLU function. The last layer is a 10-unit softmax output layer. For the MLP, the number of FLOPs required to one data sample for gradient calculation is equal to its parameters' number. In our proposed FL approach, devices only share parameters of the first 2 layers, which has 533504 parameters, accounting for 96.7\% of the entire model parameters.
    For the CIFAR-10 dataset, we train a CNN with the following structure: two $5\times5$ convolution layers each with 64 channels and followed by a $2\times2$ max-pooling layer; three fully connected layers with 1600, 120, and 64 units, respectively; and a 10-unit softmax output layer. Each convolution or fully connected layer is activated by the ReLU function. The CNN possesses 307842 parameters and our proposed FL approach only share the first 4 layers in the training process, which has 99.7\% of the total number of model parameters.
    For both MLP and CNN, the learning rate $\eta_u$ and $\eta_v$ are set to 0.05, a momentum of 0.9 is adopted, the number of local iterations is set to 5, each parameter is quantitated as 16 bits, and cross entropy is adopted as the loss function.

\item System setting: If not specified, the system parameters related to communication and computation are set as follows.
      We consider that $K=100$ devices are randomly distributed within a 500m $\times$ 500m single cell with total bandwidth $B=10$ MHz, and the PS is located in the cell's centre. The channel noise power spectral density $N_0$ is set to $-174$ dBm. For all devices in the system, we set their maximum transmit power and CPU frequency as $f_{k,\max}=1$GHz and $p_{k,\max}=1$W, respectively.
       Similar to \cite{miettinen2010energy, 9862981}, we set the energy coefficient $\kappa=5\times 10^{-27}$. The channel gain is modeled as $h_{k,t} = h_0{\rho_k}(t)(d_0/d_k)^v$, where $h_0=-30$dB is the path loss constant; $d_k$ is the distance between device $k$ and the PS; $d_0=1$m is the reference distance; $\rho_k(t)\sim \text{Exp}(1)$ is exponentially distributed with unit mean, which represents the small-scale fading channel power gain from the device $k$ to the PS in round $t$; $d_0/d_k$ represents the large-scale path loss with $v=2$ being the path loss exponent. Besides, we set $T_{\max}=2$s and $\bar{E}_k=0.1$J for the MNIST dataset, and $T_{\max}=14$s and $\bar{E}_k=2$J for the CIFAR-10 dataset.
\end{enumerate}

\begin{figure*}
\centering
\subfigure[]{\label{fig:mnistacc}
\includegraphics[width=0.43\linewidth]{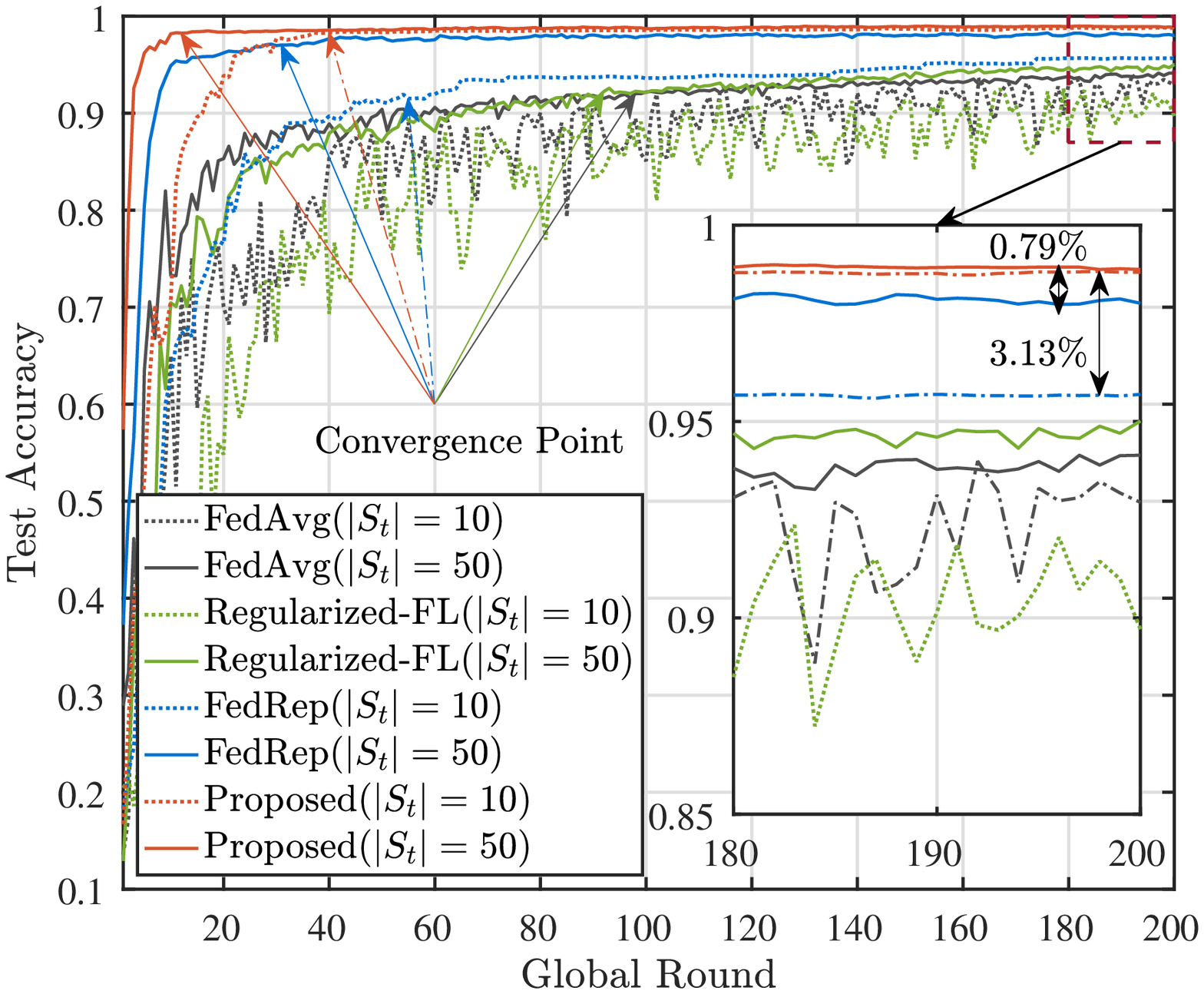}}
\hspace{0.01\linewidth}
\subfigure[]{\label{fig:mnistloss}
\includegraphics[width=0.43\linewidth]{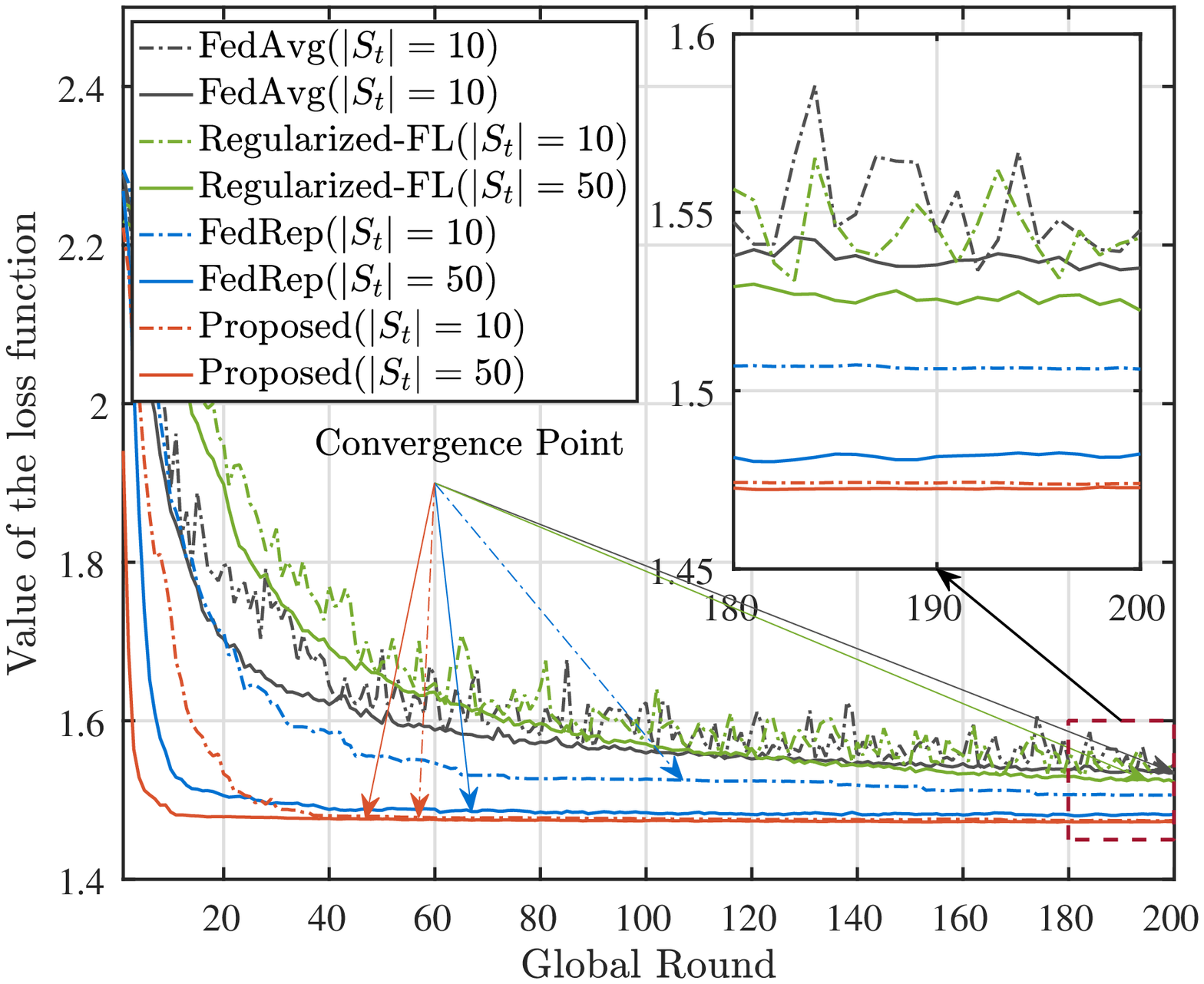}}
\vspace{-0.4cm}
\caption{Learning performance of the proposed partial aggregation approach and benchmarks on MNIST dataset: (a) test accuracy; (b) loss value.}
\label{fig:mnist}
\vspace{-0.8cm}
\end{figure*}
\vspace{-0.6cm}
\subsection{Performance of Partial Model Parameters Aggregation}
\begin{figure*}
\centering
\subfigure[]{\label{fig:cifar10acc}
\includegraphics[width=0.43\linewidth]{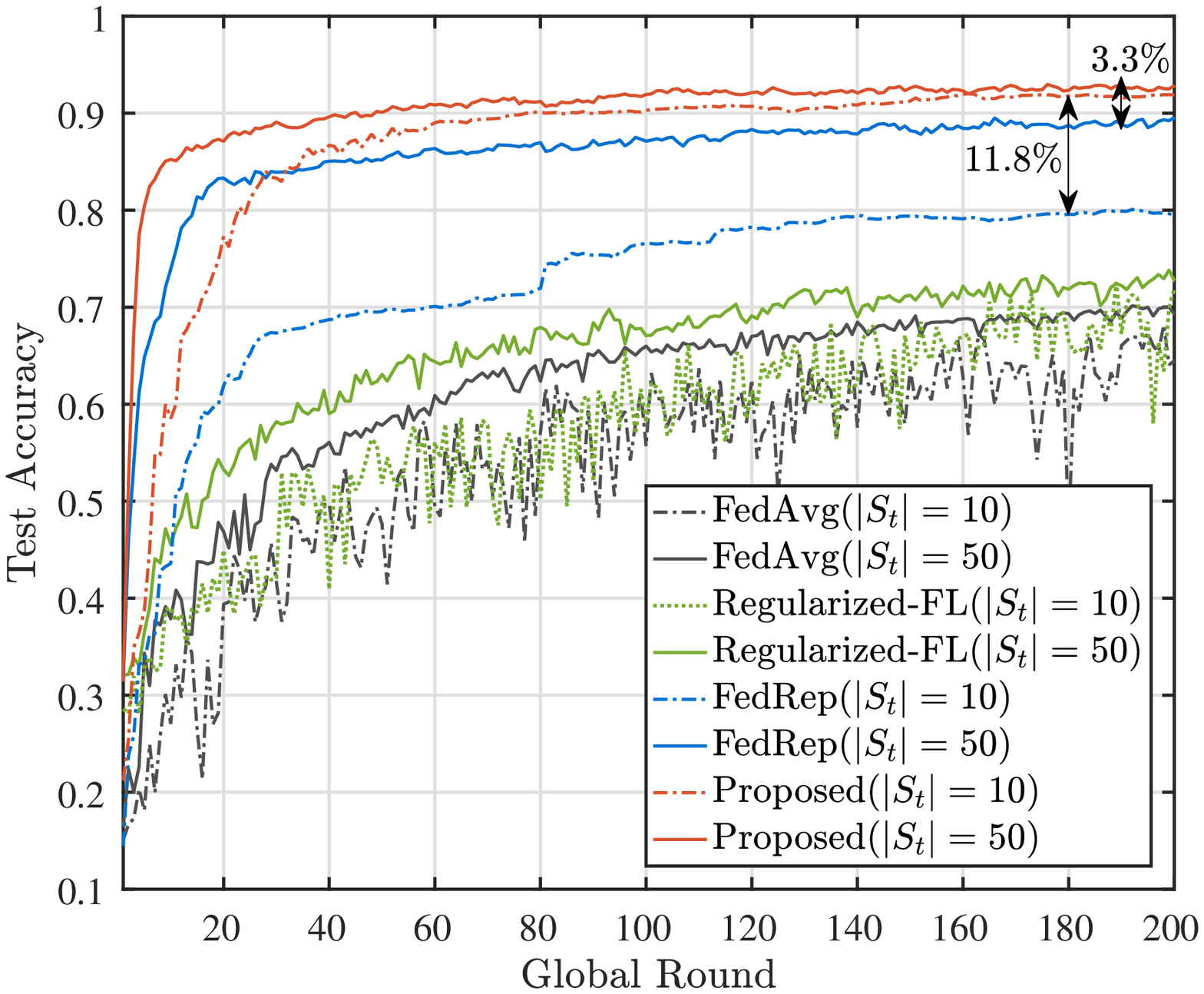}}
\hspace{0.01\linewidth}
\subfigure[]{\label{fig:cifar10loss}
\includegraphics[width=0.43\linewidth]{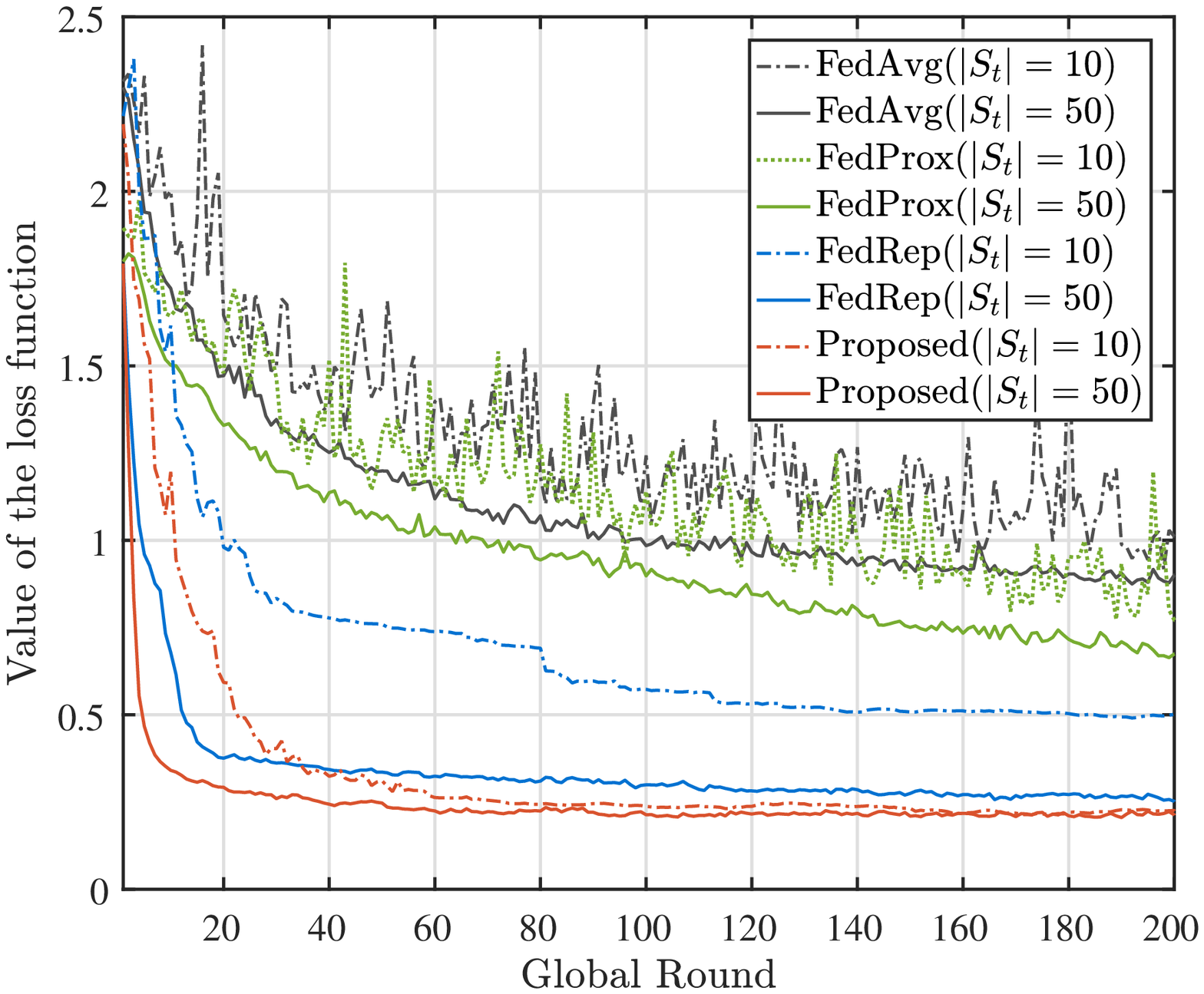}}
\vspace{-0.4cm}
\caption{Learning performance of the proposed partial aggregation approach and benchmarks on CIFAR-10 dataset: (a) test accuracy; (b) loss value.}
\label{fig:cifar10}
\vspace{-0.8cm}
\end{figure*}

To verify the advantages of the proposed PMA-FL algorithm, we compare its performance with three benchmarks.
1) \emph{Regularized FL} \cite{9187874}: Regularized FL uses a proximal term to regularize each local loss function for tackling the data heterogeneity.
2) \emph{FedAvg} \cite{pmlr-v54-mcmahan17a}: The selected devices upload the entire model to the PS for aggregation in each round.
3) \emph{FedRep} \cite{pmlr-v139-collins21a}: In each round, the selected devices sequentially train the feature extractor and predictor. After local training, the selected devices upload their feature extractors for aggregation.
Actually, Regularized FL and FedAvg requires more computation and bandwidth resources than the proposed approach. Note that, we do not consider the energy and bandwidth limitation in this subsection.

Fig. \ref{fig:mnist} compares the performance of the proposed approach with two benchmarks on the MNIST dataset. It is observed that the proposed FL approach outperforms the benchmarks in terms of test accuracy and test loss. Specifically, the proposed approach boosts 3.13\% when $\left| \bm{S}_t \right|=10$ and 0.79\% accuracy when $\left| \bm{S}_t \right|=50$ compared with the benchmark approaches.
Moreover, the proposed approach converges faster than the benchmarks. Note that the convergence point in \ref{fig:mnistacc} and \ref{fig:mnistloss} are defined as the first point that the variation of test accuracy and loss value is less than $10^{-6}$, respectively.
Additionally, compared with the three benchmarks, the proposed approach is less sensitive to the fraction of participating devices in each round. After 40 global rounds, the proposed approach with 10 devices participating in each round can obtain a similar performance as 50 devices participating in each round. The device participating ratio only affects the convergence speed and almost without reducing the final accuracy. However, the benchmarks are sensitive for the fraction of participating devices in each round, especially the training processes of Regularized FL and FedAvg are unstable when the participating ratio of devices is small, like 10 devices.

Fig. \ref{fig:cifar10} presents the performance of the proposed approach and two benchmarks on the CIFAR-10 datasets, drawing a similar conclusion with the experiments on the MNIST dataset. In particular, the proposed approach obtained a more distinct performance improvement on this more complicate dataset, boosting 11.8\% and 3.3\% accuracy than the benchmark schemes when $\left| \bm{S}_t \right|=10$ and $\left| \bm{S}_t \right|=50$, respectively. Similarly, the learning processes of Regularized FL and FedAvg is unstable when a small fraction of devices participate in each round, i.e., $\left| \bm{S}_t \right|=10$. These results indicate that the proposed approach is more robust, performing well in real datasets.

\subsection{Performance of the Proposed Energy-Efficient Device Scheduling Algorithm}
In this subsection, we verify the effectiveness of the proposed dynamic device scheduling algorithm by comparing it with the following device scheduling schemes. For fairness, we use these benchmark schemes to schedule devices for the proposed FL approach instead of their original FedAvg approach. Each curve is averaged over 100 and 50 runs for MNIST and CIFAR-10, respectively.
\begin{enumerate}[fullwidth,itemindent=1em,label=\arabic*)]
  \item \emph{Random scheduling without energy limitation (RS-WEL)}: Devices do not have energy limitation while the bandwidth and delay constraints exist. In each round, RS-WEL uses the set expansion algorithm to schedule devices. Specifically, it incrementally adds devices (randomly selected from all devices without replacement) into the scheduling set until violating the bandwidth constraint. Then, the last scheduling set that satisfies bandwidth constraints is the true scheduling device set.
  \item \emph{OCEAN}\cite{9237168}: The OCEAN is also a Lyapunov optimization-based device scheduling approach, in which the spectral bandwidth is orthogonally allocated to the scheduled devices for global aggregation in each communication round.
\end{enumerate}

Based on the MNIST dataset, Fig. \ref{fig:mnist_Ek} shows the effect of devices' energy budget on the training performance of the proposed dynamic device scheduling algorithm and two benchmarks. The results indicate that our proposed dynamic device scheduling algorithm outperforms the two benchmarks. Given the same energy budget, i.e., $\bar{E}_k=0.14$J, the proposed algorithm achieves 3.28\% test accuracy improvement comparing with the OCEAN algorithm. Moreover, the proposed algorithm is able to obtain better performance than the OCEAN algorithm under less energy budget. Specifically, the proposed algorithm with energy budget $\bar{E}_k=0.1$J (71\% of the energy budget of OCEAN) remains improving 2.59\% accuracy compared to the OCEAN algorithm with energy budget $\bar{E}_k=0.14$J.
Compared with the RS-WEL scheme with unlimited energy budget, the proposed algorithm achieves a slight accuracy improvement when the energy budget is $\bar{E}_k=0.14$J.
\begin{figure*}
\centering
\subfigure[]{\label{fig:mnist_Ek}
\includegraphics[width=0.43\linewidth]{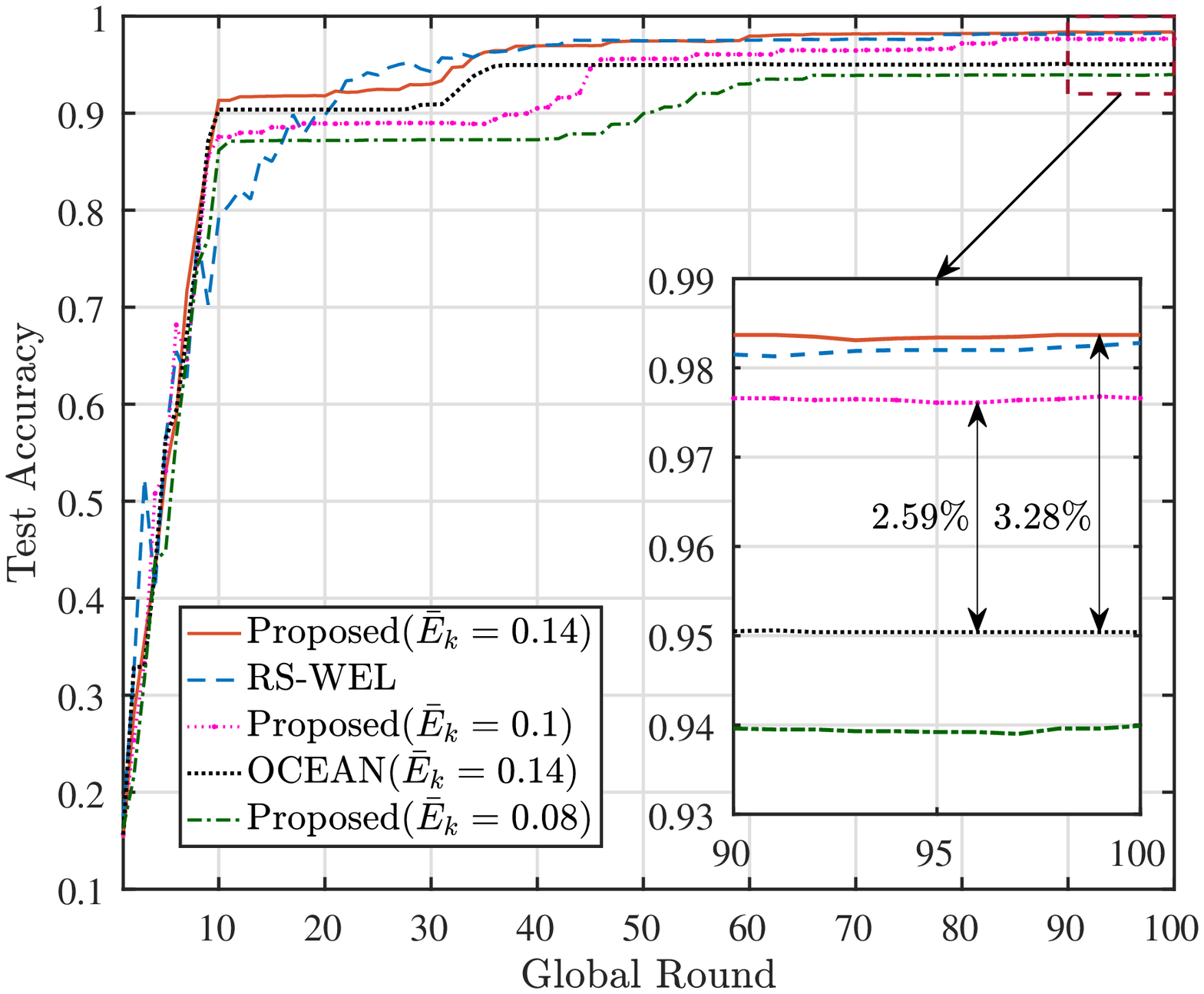}}
\hspace{0.01\linewidth}
\subfigure[]{\label{fig:cifar10_Ek}
\includegraphics[width=0.43\linewidth]{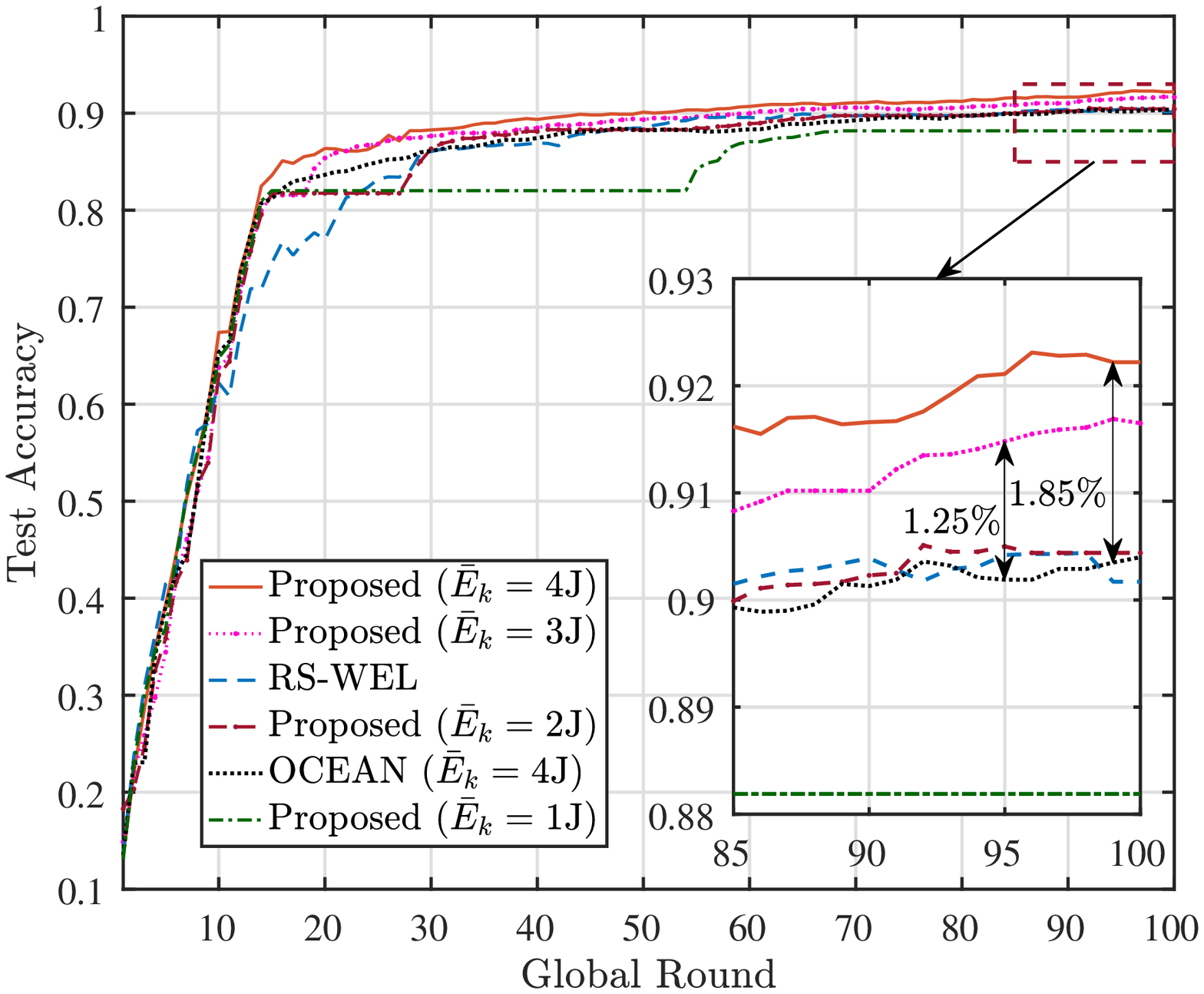}}
\vspace{-0.4cm}
\caption{Performance of the proposed algorithm and benchmarks under different energy budget $\bar{E}_k$: (a) on MNIST dataset; (b) on CIFAR-10 dataset.}
\label{fig:acc_Ek}
\vspace{-0.7cm}
\end{figure*}

A similar evaluation is made on the CIFAR-10 dataset in Fig. \ref{fig:cifar10_Ek}. Given energy budget $\bar{E}_k=4$J for both the proposed algorithm and the OCEAN algorithm, the proposed algorithm achieves around a 1.85\% accuracy boosts for the OCEAN algorithm. Similarly, the proposed algorithm under 75\% energy budget  ($\bar{E}_k=3$J) outperforms the OCEAN algorithm with an energy budget $\bar{E}_k=4$J, obtaining 1.25\% accuracy gain.
Additionally, the proposed algorithm with $\bar{E}_k=2$J obtains a similar performance as the OCEAN algorithm with $\bar{E}_k=4$J and the RS-WEL scheme.
The performance gain mainly comes from the joint optimization for both computation and wireless resources. In our proposed algorithm, the participating devices can get a trade-off between computation and communication energy consumption, achieving the most energy-efficient learning process. Specifically, the devices with poor channel conditions can boost their CPU frequency for reducing computation time and thus reserve more time for wireless communications. In contrast, devices with good channel conditions can lower the CPU frequency to balance computation and communication energy consumption.

\begin{figure*}
\centering
\subfigure[]{\label{fig:mnist_Tmax}
\includegraphics[width=0.43\linewidth]{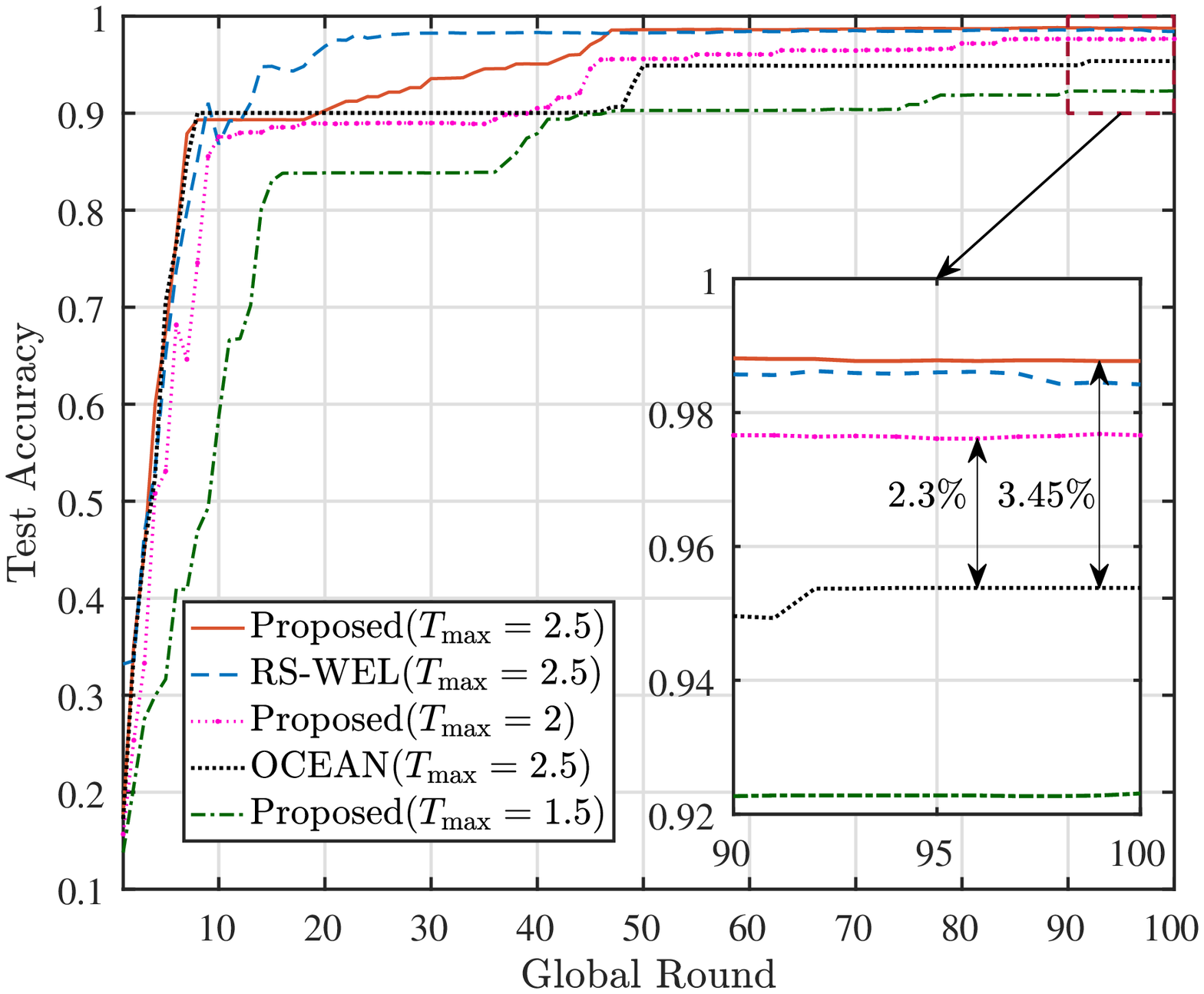}}
\hspace{0.01\linewidth}
\subfigure[]{\label{fig:cifar10_Tmax}
\includegraphics[width=0.43\linewidth]{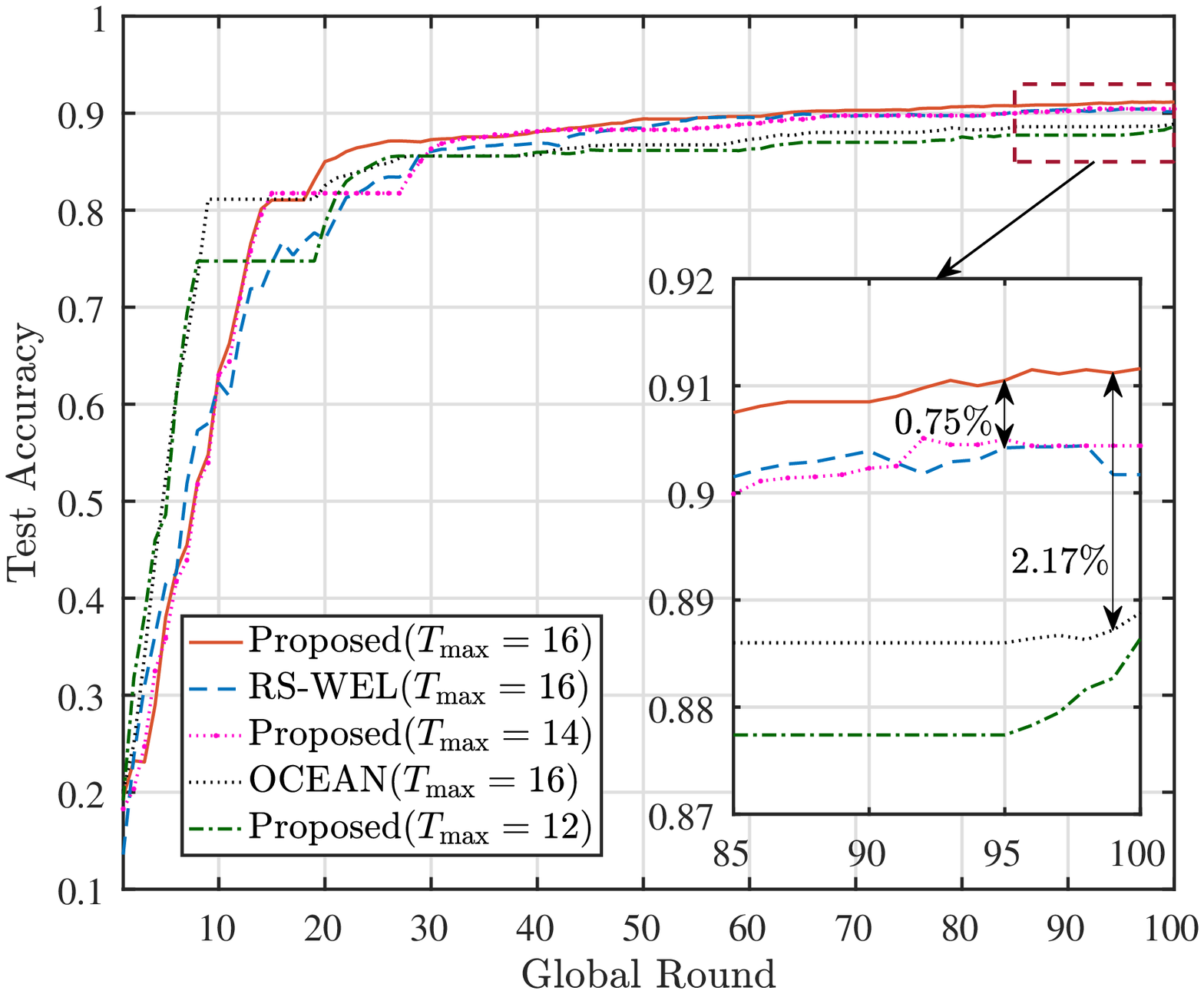}}
\vspace{-0.4cm}
\caption{Performance of the proposed algorithm and benchmarks under different delay constraint $T_{\max}$: (a) on MNIST dataset; (b) on CIFAR-10 dataset.}
\label{fig:acc_Tmax}
\vspace{-0.8cm}
\end{figure*}
We compare our proposed device scheduling algorithm with the benchmarks under different latency constraints on MNIST dataset in Fig. \ref{fig:mnist_Tmax}. Clearly, as the latency constraint, $T_{\max}$, increases, the learning performance is improved. This is because a larger $T_{\max}$ helps save the computation and communication energy and thus more data samples are able to scheduled in each round.
Moreover, with the same latency constraints, i.e., $T_{\max}=2.5$s, the proposed algorithm boosts 3.45\% test accuracy compared with the OCEAN algorithm. Using the RS-WEL as the baseline, the proposed algorithm obtains a minor accuracy gain. One interesting phenomenon is that the proposed algorithm outperforms the OCEAN algorithm with a stricter delay restriction. Specifically, given time budget $T_{\max}=2$s for the proposed algorithm, it obtains 2.3\% accuracy gain than the OCEAN algorithm with $T_{\max}=2.5$s. In other words, the proposed algorithm is able to obtain a better accuracy with a 20\% time budget reduction.
Although the proposed algorithm with $T_{\max}=1.5$s performs not good as the OCEAN algorithm with $T_{\max}=2.5$s, the above results illustrate it has the ability to improve accuracy with a stringent delay.

Fig. \ref{fig:cifar10_Tmax} shows the impact of time budget on CIFAR-10 dataset, obtaining a similar results on the MNIST dataset. Specifically, the proposed algorithm boosts 2.17\% test accuracy with the OCEAN algorithm under same delay restriction $T_{\max}=16$s. Compared with the RS-WEL scheme, the proposed algorithm gains 0.75\% performance improvement with $T_{\max}=16$s, and obtains a similar performance with $T_{\max}=14$s.
Moreover, under a stringent delay requirement, i.e., $T_{\max}=14$s, the proposed algorithm achieves a better performance than the OCEAN algorithm with $T_{\max}=16$s. That is, the proposed algorithm is able to get a better performance as the OCEAN algorithm with 12.5\% time budget reduction.
The underlying reason is that the joint optimization of computation and communication achieves lower energy consumption than solely considering the optimal communication. Even with less time budget, the balance between computation and communication can also lower the overall energy consumption, enabling more devices to participate in the FL training process in a sustainable way.

\begin{figure*}
\centering
\subfigure[]{\label{fig:mnist_cifar_V_acc}
\includegraphics[width=0.43\linewidth]{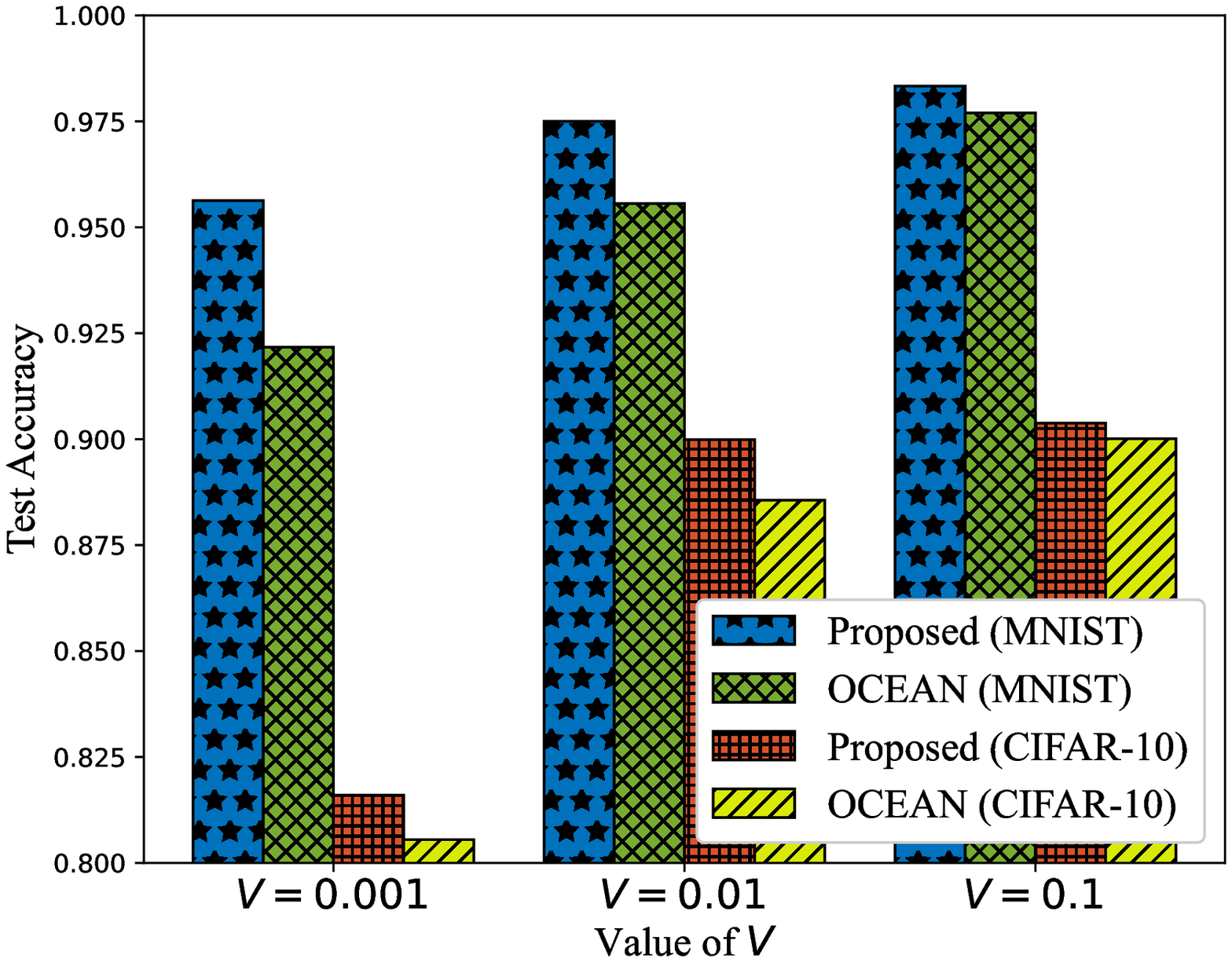}}
\hspace{0.01\linewidth}
\subfigure[]{\label{fig:mnist_cifar_ee}
\includegraphics[width=0.43\linewidth]{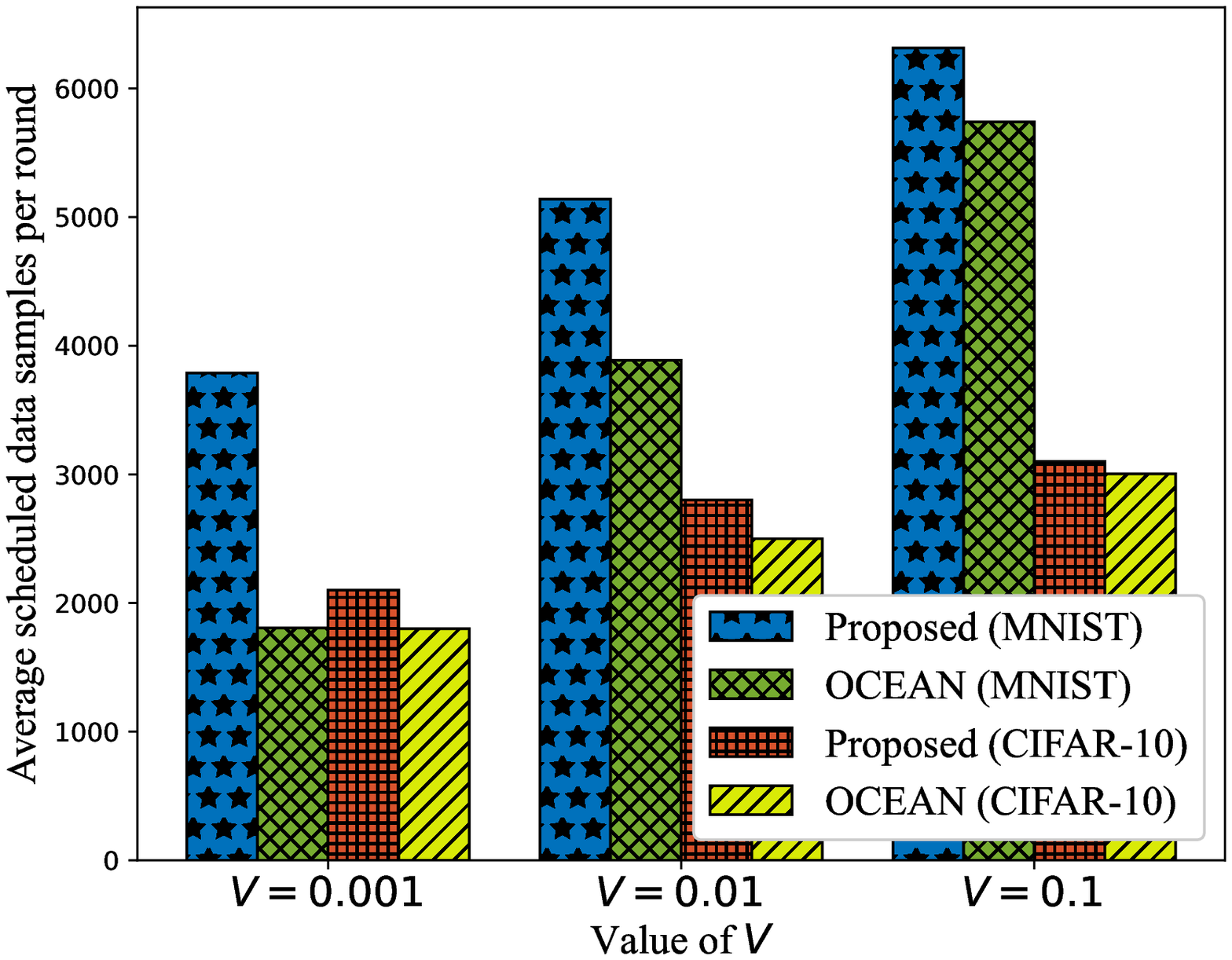}}
\vspace{-0.4cm}
\caption{Performance of the proposed algorithm and benchmarks under different weight parameter $V$: (a) test accuracy on MNIST and CIFAR-10 datasets; (b) average scheduled data samples per round.}
\label{fig:acc_V}
\vspace{-0.9cm}
\end{figure*}
In Fig. \ref{fig:acc_V}, we verify that the adjustable weight parameter $V$ is able to balance the training performance and energy consumption of devices.
Fig. \ref{fig:mnist_cifar_ee} shows that as $V$ increases, devices consume energy in a more aggressive manner, resulting in scheduling more data samples, thus obtaining accuracy improvement.
From Fig. \ref{fig:mnist_cifar_V_acc}, the experiments on the MNIST dataset indicate that the proposed algorithm achieves 3.46\%, 1.94\%, and 0.63\% test accuracy improvement compared with the OCEAN algorithm under $V=0.001$, $V=0.01$, and $V=0.1$, in each one respectively.
Interestingly, the proposed algorithm with $V=0.001$ obtains a similar performance with the OCEAN algorithm with $V=0.01$. This further reveals that the proposed algorithm has the ability to obtain a similar performance as the OCEAN algorithm under a more rigid energy restriction.
Similarly, on the CIFAR-10 dataset, the proposed algorithm boosts 1.05\% and 1.23\% accuracy in terms of $V=0.001$ and $V=0.01$, and obtains a slight accuracy improvement when $V=0.1$ compared with the OCEAN algorithm.
Note that, if $V$ is too large, the device scheduling algorithm would pay less attention for devices' energy consumption and try to schedule more devices. This may break the energy budget limitation for devices. Thus, the value of $V$ should be judiciously adjusted to optimize the training performance while satisfying the long-term energy constraints.
\vspace{-0.4cm}
\section{Conclusion}\label{sec:conclusion}
In this work, we have proposed a novel PMA-FL algorithm, which only shares the feature extractor part of neural networks for global aggregation in the learning process while the predictor part of each device is localized for personalization.
This design effectively improves the robustness and performance of the training process, overcoming the data heterogeneity across devices. Experiments show that PMA-FL is able to boost 3.13\% and 11.8\% accuracy on MNIST and CIFAR-10 datasets compared to the benchmark approaches, respectively.
In addition, we have theoretically analyzed the convergence bound of PMA-FL with a general non-convex loss function setting.
To implement the PMA-FL in resource-limited wireless networks, we have devised a joint device scheduling, communication and computation resource allocation approach to improve the learning performance by achieving the energy consumption balance between communication and computation for each device and the energy consumption-bandwidth balance between devices.
Compared with the considered benchmarks with the same energy and time budgets, PMA-FL obtained around 3\% and 2\% accuracy improvement on the MNIST and CIFAR-10 datasets, respectively. Moreover, PMA-FL is able to obtain slightly higher accuracy than the benchmarks with 29\% energy or 20\% time reduction on the MNIST; and 25\% energy or 12.5\% time reduction on the CIFAR-10.

\scriptsize
\bibliographystyle{IEEEtran}
\bibliography{IEEEabrv,cited}
\end{document}